\documentclass[10pt,twocolumn,letterpaper]{article}

\usepackage{iccv}              %

\def\paperTitle{Forecasting Continuous Non-Conservative Dynamical Systems in SO(3)}
\def\paperID{8255}
\def\confName{ICCV}
\def\confYear{2025}

\def\authorBlock{
    Lennart Bastian\textsuperscript{1,2*} \qquad
    Mohammad Rashed\textsuperscript{1,2*} \qquad
    Nassir Navab\textsuperscript{1,2} \qquad
    Tolga Birdal\textsuperscript{3} \\

    \\
    \textsuperscript{1} Technical University of Munich \;
    \textsuperscript{2} Munich Center of Machine Learning \;
    \textsuperscript{3} Imperial College London
}

\makeatletter
\DeclareRobustCommand\onedot{\futurelet\@let@token\@onedot}
\def\@onedot{\ifx\@let@token.\else.\null\fi\xspace}

\makeatother

\newcommand{\SO}{\ensuremath{\mathrm{SO}(3)}}
\newcommand{\so}{\ensuremath{\mathfrak{so}(3)}}
\newcommand{\SE}{\ensuremath{\mathrm{SE}(3)}}

\newcommand{\R}{\mathbb{R}}

\newcommand{\bR}{\mathbf{R}}
\newcommand{\bI}{\mathbf{I}}
\newcommand{\bJ}{\mathbf{J}}
\newcommand{\bomega}{\boldsymbol{\omega}}

\newcommand{\bA}{\mathbf{A}}
\newcommand{\bb}{\mathbf{b}}
\newcommand{\bD}{\mathbf{D}}
\newcommand{\bx}{\mathbf{x}}
\newcommand{\bxtilde}{\mathbf{\tilde{x}}}

\newcommand{\bW}{\mathbf{W}}

\newcommand{\hatop}[1]{\hat{#1}}

\newcommand{\phit}{\varphi(t)}
\newcommand{\bp}{\mathbf{p}}
\newcommand{\ftheta}{f_\theta}
\newcommand{\gtheta}{g_{\theta'}}
\newcommand{\zetatheta}{\zeta_\theta}
\newcommand{\expmap}{\mathrm{Exp}}
\newcommand{\logmap}{\mathrm{Log}}

\newcommand{\tk}{t_k}
\newcommand{\thetaparam}{\theta}

\newcommand{\brho}{\boldsymbol{\rho}}
\newcommand{\brhok}{\boldsymbol{\rho}_k}
\newcommand{\bz}{\mathbf{z}}

\newcommand{\Xpath}{X}
\newcommand{\dX}{\mathrm{d}X}
\newcommand{\ddX}{\mathrm{d}^2X}

\newcommand{\Nsamp}{N}

\newcommand{\bxtildek}{\mathbf{\tilde{x}}_k}

\usepackage{graphicx}	
\usepackage{amsmath}	
\usepackage{amssymb}	
\usepackage{booktabs}
\usepackage{times}
\usepackage{microtype}
\usepackage{epsfig}
\usepackage{caption}
\usepackage{colortbl}
\usepackage{enumitem}
\usepackage{tabularx}
\usepackage{xstring}
\usepackage{multirow}
\usepackage{xspace}
\usepackage{url}
\usepackage{subcaption}
\usepackage{footnote}                %
\usepackage{marvosym} %

\usepackage{cuted}

\usepackage{comment}
\usepackage{amsthm}
\usepackage{multirow}
\usepackage{refcount}

\usepackage[utf8]{inputenc}
\usepackage{makecell} 

\usepackage{xr-hyper}
\usepackage{wrapfig}  %
\usepackage{etoc}     %

\usepackage{thmtools}

\declaretheoremstyle[
  spaceabove=6pt,
  spacebelow=6pt,
  headfont=\normalfont\bfseries,
  notefont=\mdseries,
  notebraces={(}{)},
  bodyfont=\normalfont\itshape,
  postheadspace=1em
]{mytheoremstyle}

\declaretheorem[style=mytheoremstyle,name=Theorem]{theorem}
\declaretheorem[style=mytheoremstyle,name=Proposition,numberlike=theorem]{proposition}
\declaretheorem[style=mytheoremstyle,name=Definition,numberlike=theorem]{definition}
\declaretheorem[style=mytheoremstyle,name=Lemma,numberlike=theorem]{lemma}
\declaretheorem[style=mytheoremstyle,name=Proposition,numbered=no]{proposition*}

\AtEndPreamble{
    \crefname{theorem}{Thm.}{Thms.}
    \crefname{proposition}{Prop.}{Props.}
    \crefname{definition}{Def.}{Defs.}
    \crefname{lemma}{Lem.}{Lems.}

    \creflabelformat{theorem}{(#2#1#3)}
    \creflabelformat{proposition}{(#2#1#3)}
    \creflabelformat{definition}{(#2#1#3)}
    \creflabelformat{lemma}{(#2#1#3)}
}

\renewcommand{\paragraph}[1]{{\vspace{0.6mm}\noindent \bf #1}.}

\newcommand{\App}{App.~}

\definecolor{iccvblue}{rgb}{0.21,0.49,0.74}
\usepackage[pagebackref,breaklinks,colorlinks,allcolors=iccvblue]{hyperref}

\usepackage[capitalize]{cleveref}
\crefname{section}{Sec.}{Secs.}
\Crefname{section}{Section}{Sections}
\Crefname{table}{Table}{Tables}
\crefname{table}{Tab.}{Tabs.}

\def\paperID{8255} %
\def\confName{ICCV}
\def\confYear{2025}

\begin{document}
\title{\paperTitle}
\author{\authorBlock}

\maketitle
{\let\thefootnote\relax\footnote{{
\vspace{-0.2cm}
\parbox{\linewidth}{%
$^*$Equal Contribution\\
\raisebox{-0.2ex}{\Letter} \ : {\tt \scriptsize lennart.bastian (at) tum.de}
}}}}

\begin{abstract}
Modeling the rotation of moving objects is a fundamental task in computer vision, yet $SO(3)$ extrapolation still presents numerous challenges: (1) unknown quantities such as the moment of inertia complicate dynamics, (2) the presence of external forces and torques can lead to non-conservative kinematics, and (3) estimating evolving state trajectories under sparse, noisy observations requires robustness.
We propose modeling trajectories of noisy pose estimates on the manifold of 3D rotations in a physically and geometrically meaningful way by leveraging Neural Controlled Differential Equations guided with $SO(3)$ Savitzky-Golay paths.
Existing extrapolation methods often rely on energy conservation or constant velocity assumptions, limiting their applicability in real-world scenarios involving non-conservative forces. 
In contrast, our approach is agnostic to energy and momentum conservation while being robust to input noise, making it applicable to complex, non-inertial systems. 
Our approach is easily integrated as a module in existing pipelines and generalizes well to trajectories with unknown physical parameters.
By learning to approximate object dynamics from noisy states during training, our model attains robust extrapolation capabilities in simulation and various real-world settings.
Code is available \href{https://github.com/bastianlb/forecasting-rotational-dynamics}{here}.
\end{abstract}

\vspace{-5mm}
\section{Introduction}

Modeling the dynamics of rigid bodies in 3D is fundamental to numerous domains such as tracking, robotic control, trajectory planning, and sensor fusion~\cite{fan2025benchmarks,duong2021hamiltonian,mason2023learning,cho2014multi}.
While the mathematical foundations of rigid body mechanics are well understood, many modern applications require predicting rotational motion from noisy sensor data when key physical parameters are unknown. 
This presents significant challenges at the intersection of data-driven and physical modeling.

The core difficulty stems from the inherent nonlinearity of rotational motion in \SO, where even simple rigid bodies can exhibit complex behaviors ranging from regular precession\footnote{consider a spinning top balancing due to the gyroscopic effect.} to chaos~\cite{borisov2018rigid}.
Unlike translational motion in Euclidean space, rotational dynamics are governed by the moment of inertia tensor - a quantity that couples the body's mass distribution with its angular momentum in non-trivial ways. 
When this inertial property is unknown, direct estimation becomes an ill-posed inverse problem, particularly in the presence of external torques and measurement noise~\cite{mason2023learning,park2008kinematic,gugushvili2012consistent}.
This problem is further complicated by the system's sensitivity to initial conditions - a characteristic of chaotic behavior that emerges in many rotating rigid body systems~\cite{borisov2018rigid}.

\begin{figure}[t!]
    \includegraphics[width=\columnwidth]{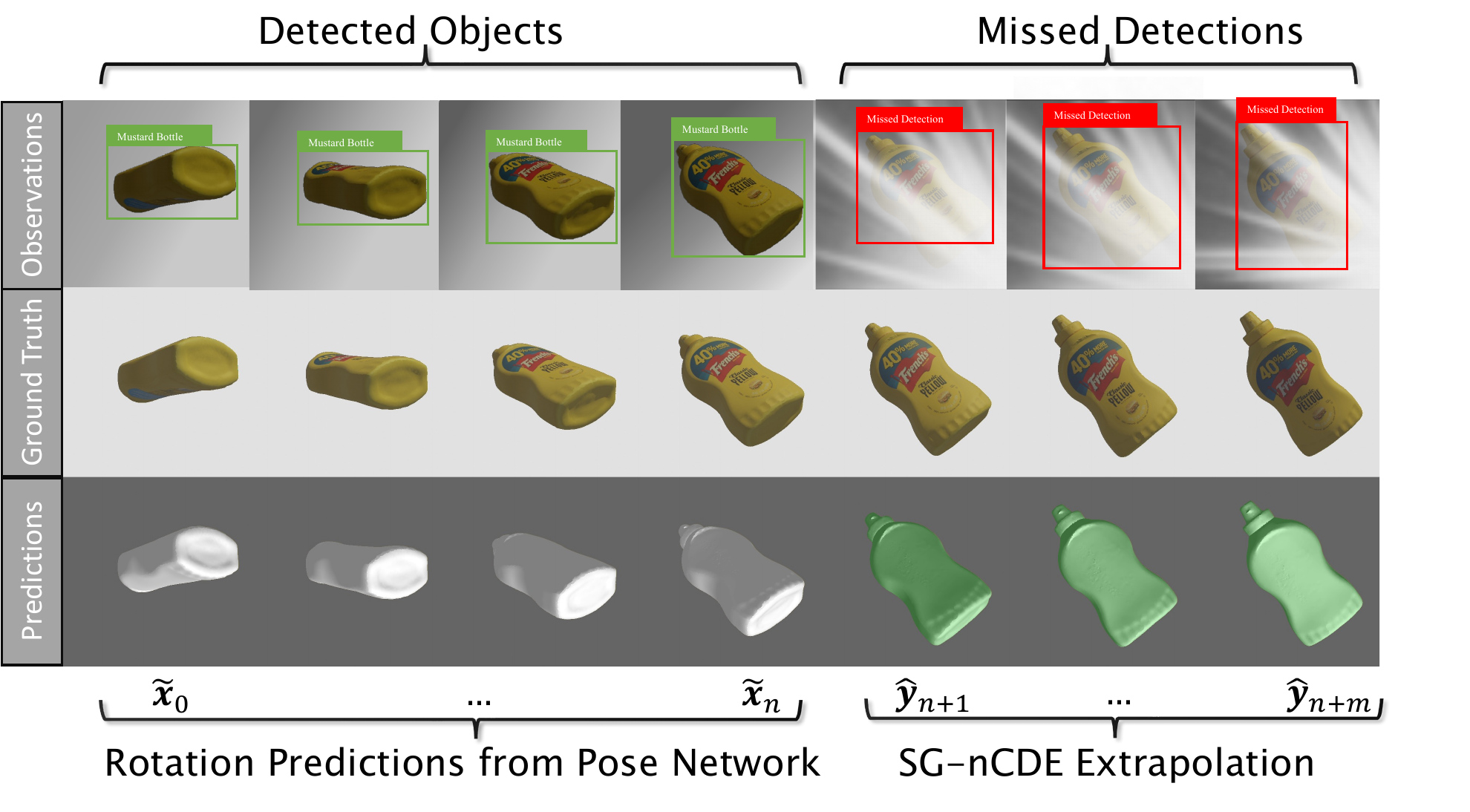}
    \vspace{-6mm}
    \caption{We propose Savitzky-Golay Neural Controlled Differential Equations (SG-nCDEs) for \SO~forecasting, a fundamental task at the intersection of tracking and object pose estimation.
    Real-world rotational observations can be occluded or otherwise errant.
    Above, pose predictions (bottom left) are input into an SG-nCDE, which robustly compensates for missed detections of a rotational object trajectory (green predictions, bottom right).}
    \label{fig:object_pose}
    \vspace{-5.0mm}
\end{figure}

This work proposes simultaneously addressing \SO~state estimation and learning unknown physical parameters from noisy measurements.
In contrast to previous approaches that restrict themselves to conservative systems, our method can predict rotational object trajectories in various physical scenarios, such as dissipative forces, or for objects under the influence of an external conservative force field.
This versatility opens the door to a variety of applications, from improving the situational awareness of autonomous systems \cite{ettinger2021large,mason2023learning}, compensating for missing sensor measurements in tracking and localization~\cite{fan2025benchmarks,bosse2009continuous}, to dynamic grasping~\cite{marturi2019dynamic,wu2022grasparl}.

However, forecasting rotation presents several challenges: learned methods neglecting the geometrical nuances of \SO~typically struggle due to numerical errors or non-linearities~\cite{geist2024learning,zhou2019continuity,bregier2021deepregression,chen2022projective}.
Analogously, naive learning often fails to respect fundamental physical invariants motivating physics-guided machine learning (PGML) where explicit physical priors are incorporated into neural architectures~\cite{chen2018neural,toth2019hamiltonian,cranmer2020lagrangian}.
An effective inductive bias should address both paradigms, learning the underlying physical laws and respecting the manifold geometry of \SO.

To this end, we contribute the following:

\begin{itemize}
    \item We propose to model the kinematic evolution of rotating bodies via a latent dynamical system, posing neural controlled differential equations (neural CDEs) as an appealing modeling paradigm for this setting.
    \item We design an integration control signal on the Lie group \SO~to simultaneously handle denoising input states and serve as a numerical integration path for a CDE; our model learns a solution to the underlying physics based on a path in \SO.
    This path is demonstrably suitable for controlling a neural CDE.
    \item Extensive experimental validation confirms our hypothesis that we can approximate a dynamical system with unknown inertial properties and external forces; the proposed method surpasses existing methods in both extrapolation capabilities and robustness to noise. 
    We demonstrate the practicality of our method in simulation and various real-world object-tracking settings.
\end{itemize}

\section{Related Works}

We first review the relevant literature on \SO{} filtering, representation learning, and PGML.

\paragraph{SO(3) Filtering} 
Robustly filtering measurements from visual or inertial sensors is essential for many applications from video stabilization \cite{jia2014constrained}, continuous-time state estimation in robotics \cite{tang2019white,wong2020data,talbot2024continuous}, tracking tasks \cite{fan2025benchmarks}, sensor fusion~\cite{nirmal2016noise,busam2017camera}, and SLAM~\cite{bosse2009continuous,lovegrove2013spline,hug2020hyperslam}.
Filtering approaches using splines~\cite{kim1995general,lovegrove2013spline,haarbach2018survey} must be constructed to preserve manifold properties of \SO~to avoid accruing numerical errors \cite{busam2017camera,wang2023pypose}. 
However, \SO~spline derivatives cannot be solved in closed form, making their integration into robust interpolation schemes computationally intensive~\cite{sommer2020efficient}.
Without optimization, splines are unsuitable for extrapolation as they diverge near interpolation boundaries \cite{persson2021practical}; many online methods requiring state prediction (e.g., ORB-SLAM \cite{mur2015orb}) therefore rely on simplified constant-velocity assumptions.

Savitzky-Golay filtering interpolates \SO~more efficiently by regressing a polynomial in the tangent space and mapping back to the manifold via the exponential map \cite{jongeneel2022geometric}. 
This approach allows for simultaneous angular velocity and acceleration estimation solvable in a linear system, providing a robust method for processing noisy rotation data in \SO, recently exemplified for hand tracking~\cite{fan2025benchmarks}.
However, extrapolation of rotations presents challenges beyond interpolation and is particularly sensitive to error accumulation. 
Existing methods for extrapolation on \SO~have considered Gaussian Processes (GPs) for their ability to model uncertainty and handle continuous-time data \cite{lang2014gaussian,lang2017computationally,lilge2024incorporating,giacomuzzo2024black}, or regress incremental changes in \SE~directly from point cloud sequences \cite{byravan2017se3}.
However, they do not consider physical laws governing the underlying dynamics.

\paragraph{Learning Representations in SO(3)}
The literature on learning 3D rotations has explored various representations, each with distinct trade-offs. 
While classical approaches like Euler angles \cite{kundu20183d}, axis-angle, or quaternions \cite{zhao2020quaternion,he2024nrdf} are intuitive, they can suffer from singularities (e.g., gimbal lock) or have topological constraints that make them challenging to regress directly~\cite{zhou2019continuity,bregier2021deepregression,geist2024learning} and backpropagate through~\cite{chen2022projective,teed2021tangent}. 
This is fundamentally due to \SO~not being homeomorphic to any subset of 4D Euclidean space~\cite{zhou2019continuity}. 
Recent work has demonstrated that these limitations impede deep learning performance~\cite{zhou2019continuity,levinson2020analysis,geist2024learning}, motivating the development of higher-dimensional continuous representations that enable stable optimization by mapping \SO~to an overparameterized continuous space.
Notable approaches include 6D~\cite{zhou2019continuity}, 9D~\cite{levinson2020analysis}, and 10D~\cite{peretroukhin2020smooth} representations.

\paragraph{Physics-Guided Extrapolation}
Various paradigms have been devised to instill data-driven models with priors from known physical laws \cite{willard2020integrating}.
Among these are Physics-Informed Neural Networks \cite{raissi2019physics}, which construct regularizers to make network predictions physically plausible.
Another line of work seeks to learn the Hamiltonian~\cite{toth2019hamiltonian} or Lagrangian~\cite{cranmer2020lagrangian}, ensuring predictions adhere to conservation laws like energy and momentum.
While real-world settings require models to disentangle dynamics from sensor noise, these works focus on learning the physical dynamics from ground truth observations.

On the other hand, Neural Ordinary Differential Equations (neural ODEs) explicitly solve an initial value problem (IVP) parameterized by a latent representation \cite{chen2018neural}.
Several works such as GRU- and RNN-ODEs \cite{rubanova2019latent,de2019gru} have been developed to incorporate additional measurements beyond the initial value more effectively, conditioning neural ODEs with accumulated representations, making them more suitable for time series modeling.
Neural controlled differential equations (Neural CDEs)~\cite{kidger2020neural} achieve this by bridging controlled differential equations with deep neural networks. 
Instead of merely solving an IVP, the hidden state is integrated with the derivative of a control signal.

While Hamiltonian formulations have been used for state prediction from noisy observations in \SO~\cite{mason2023learning} and \SE~\cite{duong2021hamiltonian}, adapting these for non-conservative systems proves cumbersome.
\cite{mason2023learning} introduce a learning-based approach for predicting 3D rotational dynamics from images of rigid bodies with unknown mass distribution. 
They design a network that learns a latent representation homeomorphic to \SO~directly from images to extrapolate an object trajectory.
However, this approach has limitations when modeling noisy signals as it is sensitive to the initial momentum estimation and, by design, cannot handle non-conservative systems.
Furthermore, this confounds object pose estimation with learned dynamics, requiring many annotated image sequences to capture complex dynamics.
Instead, we observe that \SO~dynamics models can be trained in simulation and later coupled with specialized pose detectors \cite{ornek2025foundpose,chen2024secondpose,jung2024housecat6d}, generalizing to noisy real-world dynamics.

\begin{figure*}[t]
  \centering
  \includegraphics[width=\textwidth]{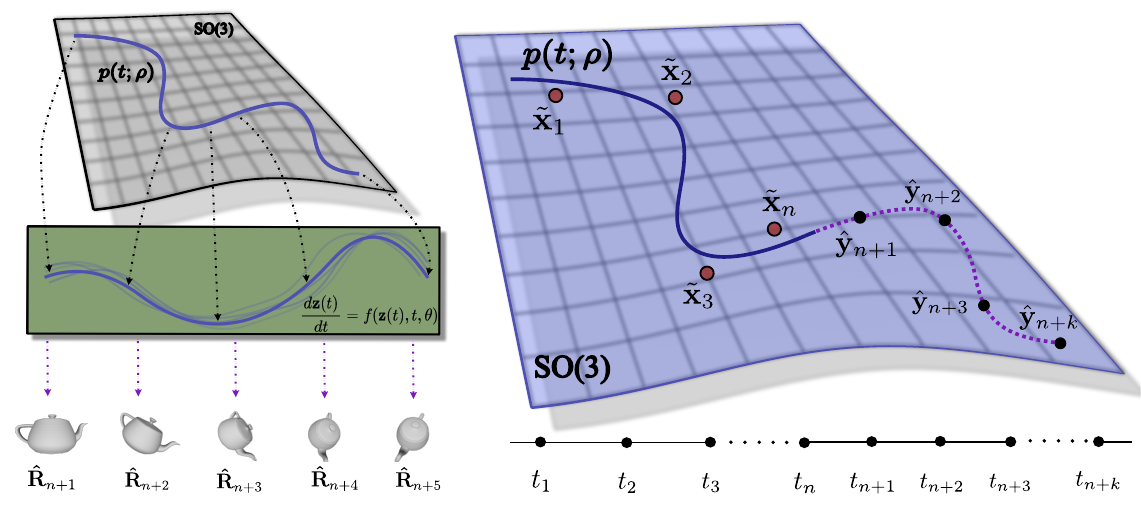}
  \caption{
    We consider the problem of forecasting the rotational motion of a rigid body with unknown physical properties from noisy sensor estimates $\mathbf{\tilde{x}}_k$ in \SO.
    The core of our proposed method is a neural controlled differential equation, which learns a latent representation of the underlying dynamical system with respect to a control path $\mathbf{p}(t; \cdot)$ defined by coefficients $\mathbf{\rho}$ (left).
    We construct this control path directly on the manifold \SO~by filtering noisy input observations with a Savitzky-Golay filter (right).
    This allows the CDE to learn a more robust parameterization for extrapolating rigid-body trajectory estimates $\mathbf{\hat{y}}_k$ than existing methods.
  }
  \label{fig:cde_main_figure}
  \vspace{-3.0mm}
\end{figure*}

\section{Preliminaries}

\paragraph{Rigid-body Dynamics}
While translation and rotation of a rigid body in $\mathbb{R}^3$ each have three degrees of freedom, the space of rotations forms a non-commutative Lie group. 
This makes modeling an evolving trajectory more challenging than the Euclidean translational counterpart.

\begin{definition}[Special Orthogonal Group]
\label{def:so3}
The special orthogonal group $\mathrm{SO}(3)$ represents the set of all rotations about the origin in $\mathbb{R}^3$:
\begin{equation}
\mathrm{SO}(3)=\left\{\mathbf{R} \in \mathbb{R}^{3 \times 3}: \mathbf{R}^{\top} \mathbf{R}=\mathbf{I}, \operatorname{det}(\mathbf{R})=1\right\}
\end{equation}
\end{definition}

\noindent Modeling kinematics also requires derivatives in $\mathrm{SO}(3)$.
The Lie algebra $\mathfrak{so}(3)$, the tangent space at the identity of $\mathrm{SO}(3)$, consists of real $3 \times 3$ skew-symmetric matrices.

\begin{definition}[Hat Operator]
\label{def:hat_operator}
Any $\bomega \in \mathbb{R}^3$ can be associated with an element $\hat{\bomega} \in \mathfrak{so}(3)$ via the hat operator:
\begin{align}
\hat{\bomega} = S(\bomega) = \begin{pmatrix}
0 & -\omega_3 & \omega_2 \\
\omega_3 & 0 & -\omega_1 \\
-\omega_2 & \omega_1 & 0 \\
\end{pmatrix}
\end{align}
where the hat operator $\hat{\bomega} : \mathbb{R}^3 \to \mathfrak{so}(3)$ maps vectors to skew-symmetric matrices.
\end{definition}

The time evolution of the rotation matrix $R(t)$ is given by $\dot{R} = R \hat{\bomega}$ for the angular velocity $\bomega$ \cite{jongeneel2022geometric}.
Here, a differential geometry perspective is useful: $\mathfrak{so}(3)$ is isomorphic to the tangent space $T_{\bR}\mathrm{SO}(3)$ at $\bR$ \cite{lee2017global}; we characterize rotation trajectories using these geometric definitions.

\section{$SO(3)$ Savitzky-Golay Neural-CDEs}
\label{sec:method}
We proceed by describing a method designed to extrapolate noisy rotation estimates of a dynamical system where the moment of inertia and external forces are unknown.
The core of the method consists of a neural CDE, which is integrated with respect to a control signal on the manifold \SO.
We construct this control signal using an \SO~Savitzky-Golay filter \cite{jongeneel2022geometric}, providing a robust de-noised prior for extrapolating the learned latent dynamical system.
An overview can be seen in \cref{fig:cde_main_figure}.

\subsection{$SO(3)$ Neural CDEs}

\begin{definition}[Discrete Trajectories in \SO]
\label{def:observations}
Let $\mathcal{T} = \{(t_k, \bxtilde_k)\}_{k=0}^{\Nsamp}$ be a sequence of observations of an object $\bx$ at discrete time points $t_k$, where $\bxtilde_k \in \SO$ represents a noisy rotational measurement at time $t_k$:
\begin{equation}
\nonumber
\mathcal{T} = \{(t_0, \bxtilde_0), (t_1, \bxtilde_1), \ldots, (t_{\Nsamp}, \bxtilde_{\Nsamp})\} \subset \R \times \SO
\end{equation}
\end{definition}

Such a sequence could come from a rotating (non-zero angular velocity) object trajectory obtained directly from images through a 6D pose estimator \cite{wang2021gdr} or in combination with other sensors.
We seek to estimate the underlying dynamics of the object in a data-driven manner to extrapolate unseen rotations for $\hat{t} > t_N$.

\begin{definition}[Neural CDE \cite{kidger2020neural}]
\label{def:neural_cde}
Let $\Xpath: [t_0, t_n] \rightarrow \R^{v+1}$ be a control path, where $\Xpath_{t_i} = (t_i, \bx_i)$. 
A Neural CDE is defined as:
\begin{equation}
\label{eq:neural_cde}
\bz_t = \bz_{t_0} + \int_{t_0}^t \ftheta(\bz_s) \dX_s \quad \text{for } t \in (t_0, t_n]
\end{equation}
where:
\begin{itemize}
\item $\bz_t \in \R^w$ is the hidden state at time $t$,
\item $\ftheta: \R^w \rightarrow \R^{w \times (v+1)}$ is a neural network with parameters $\thetaparam$,
\item $\bz_{t_0} = \zetatheta(t_0, \bx_0)$ is the initial condition, where $\zetatheta: \R^{v+1} \rightarrow \R^w$ is a network to encode the initial condition $(t_0, \bx_0)$.
\end{itemize}
\end{definition}

\begin{lemma}[Constructing Neural CDEs]
\label{lemma:cde_properties}
The solution $\bz$ to \cref{eq:neural_cde} is the response of a Neural CDE controlled by the path $\Xpath$. The control path $\Xpath$ can be constructed arbitrarily as long as it is continuous and differentiable.
\end{lemma}

\noindent Kidger et al. \cite{kidger2020neural} construct the path $\Xpath \in \mathcal{C}^2$ as a natural cubic spline over $\mathcal{T}$, while \cite{morrill2022choice} use Hermite splines for their local support.
Others have proposed learning integration paths~\cite{jhin2023learnable}.
We observe that these choices neither obey the geometric structure of \SO, nor are they robust to sensor noise (see \App B.3 for details).

\subsection{Robust $SO(3)$ Savitzky-Golay Regression}
\label{sec:robust_regression}

To address sensor noise from real-world trajectories, we propose constructing an integration path with a robust regression signal to integrate \cref{eq:neural_cde}. 
However, robust regression on $\SO$ is non-trivial due to the non-Euclidean nature of the Lie group.
Our control signal should avoid expensive optimizations or polynomial basis functions unsuited for extrapolation~\cite{sommer2020efficient}.
Savitzky-Golay filters achieve this by constructing a polynomial with global support in the Lie algebra, solvable as a single least squares problem \cite{jongeneel2022geometric}.

\begin{definition}[Lie Algebra Polynomial]
\label{def:lie_algebra_polynomial}
We define a second-order polynomial in the Lie algebra $\so$ as:
\begin{equation}
p(t; \brho) := \brho_0 + \brho_1 t + \frac{1}{2}\brho_2 t^2 \in \so
\end{equation}
where $\brho = [\brho_0; \brho_1; \brho_2] \in \R^9$ represents the vector of polynomial coefficients.
\end{definition}

\begin{definition}[Control Path on $\SO$]
\label{def:control_path}
Given the sequence of noisy observations $\mathcal{T}$ from \cref{def:observations}, we define a smooth control path $\phit \in \SO$ as:
\begin{equation}
\phit = \expmap(p(t-\tk; \brhok))\bxtildek
\end{equation}
where $\tk$ is the anchor point time, $\bxtildek \in \SO$ is the corresponding noisy measurement, $\expmap: \so \to \SO$ is the exponential map, and $\brhok$ represents the polynomial coefficients specific to time $\tk$.
\end{definition}

The control path $\phit$ provides a geometric solution to the path construction problem in \cref{lemma:cde_properties}. 
By designing $\phit$ to be smooth and respect the structure of $\SO$, we can construct a more appropriate control signal for the Neural CDE than traditional cubic or Hermite splines.

\begin{theorem}[Savitzky-Golay Filtering on $\SO$ \cite{jongeneel2022geometric}]
\label{thm:sg_optimization}
The optimal polynomial coefficients $\brhok$ at time $\tk$ can be obtained by solving the following least-squares problem:
\begin{align}
\label{eq:sg_optimization}
\brhok = \arg\min_{\brho} \sum_{m=-n}^{n} \|&(\logmap(\bxtilde_{k+m}\bxtilde_k^{-1}) \nonumber \\
&- p(t_{k+m}-\tk; \brho))\|^2
\end{align}
where $2n+1$ is the window size, and $\logmap: \SO \to \so$ maps the rotation difference to its corresponding element in the Lie algebra.

The optimization problem in \cref{eq:sg_optimization} reduces to:
\begin{equation}
\brhok = \arg\min_{\brho} \|\bA\brho - \bb\|^2
\end{equation}
with solution:
\begin{equation}
\brhok = (\bA^\top\bA)^{-1}\bA^\top\bb
\end{equation}
where $\bA \in \R^{3(2n+1) \times 9}$ and $\bb \in \R^{3(2n+1)}$ are constructed from time and rotation differences, respectively.
\end{theorem}

\noindent The constructions of $\bA$ and $\bb$ are detailed in \App C.2.

The filter described in \cref{def:control_path} can be used to numerically integrate the Neural CDE in \cref{eq:neural_cde} directly.
However, Savitzky-Golay filters exhibit artifacts near the boundary $t_N$, which can hinder extrapolating the learned system dynamics.
Therefore, we propose weighting the Savitzky-Golay trajectory with a weight matrix as proposed for the Euclidean analog \cite{schmid2022and}. 
Unlike existing works, we construct this out of learnable parameters $\mathbf{W}$, leading to the following weighted formulation:

\begin{lemma}[Weighted Savitzky-Golay Optimization]
\label{lemma:weighted_sg}
The polynomial coefficients of the weighted Savitzky-Golay filter can be computed in closed form as:
\begin{equation}
\boldsymbol{\rho} = (\mathbf{A}^\top \mathbf{W} \mathbf{A})^{-1} \mathbf{A}^\top \mathbf{W} \mathbf{b}
\end{equation}
where $\mathbf{W}$ consists of learnable parameters that can be updated through stochastic gradient descent.
\end{lemma}

\noindent Our model learns to appropriately weight this denoising signal in a manner well-suited for extrapolation while evolving the dynamical system in \cref{eq:neural_cde}.

\subsection{$SO(3)$ Savitzky-Golay Neural-CDEs}

Control paths for Neural CDEs are discussed extensively in \cite{morrill2022choice}; the authors show that a desirable control function should be sufficiently smooth and bounded to attain the universal approximation properties of Neural CDEs.

We show that $SO(3)$ Savitzky-Golay filters admit these properties in what follows.
To mitigate discontinuities and non-bounded terms in the piece-wise polynomial regression, we regress $\phit$ by defining a window over the last $2n+1$ samples of $\{\bxtildek\}_{k=0}^{\Nsamp}$. We can then observe that this global Savitzky-Golay filter attains the desirable properties of a control path for \cref{eq:neural_cde} in the following informal result, as long as the rotational difference is bounded within this interval.
This last assumption is required to avoid discontinuities of $\logmap: \SO \to \so$ at $\pi$, which we can assure with a sufficiently high sampling rate or small enough window size.

\begin{proposition}
\label{prop:control_signal}
Let $\phit$ be the control path constructed in \cref{def:control_path}. If the maximum rotational difference $\delta_n = \max_k \|\logmap(\bxtilde_k \bxtilde_0^{-1})\|$ within the window is bounded, then:
\begin{enumerate}
    \item $\phit$ is analytic and twice differentiable,
    \item The derivatives $\varphi'(t)$ and $\varphi''(t)$ are bounded,
    \item $\phit$ uniquely minimizes the problem in \cref{thm:sg_optimization}.
\end{enumerate}
\end{proposition}

\begin{proof}[Proof sketch]
As $\phit$ is defined as a composition of polynomial and analytic functions on $\SO$ (see \cref{def:control_path}), 
we observe that the first and second derivatives also take this form.  
Boundedness can be shown by analyzing the construction of the polynomial coefficients $\brho$ with respect to $\delta_n$ and considering the derivatives of $\phit$. 
Uniqueness follows from strict convexity of the optimization problem in~\cref{thm:sg_optimization}.
\end{proof}

We treat this thoroughly in \App D.
This result assures that the approximated evolving trajectory $\phit$ is well behaved as a control path for \cref{eq:neural_cde}.
In addition to fulfilling the properties of a suitable control path for a CDE, this approach is more robust for extrapolation than hermite splines and allows more stable numerical integration (see \cref{sec:experimental_main}).

\begin{table*}[ht!]
\resizebox{\textwidth}{!}{
   \setlength{\tabcolsep}{7pt}
   \small
   \centering
  \begin{tabular}{c|lccccc}
    \toprule
    \makecell{Extrapolation \\ Horizon} & Method & \makecell{Freely\\ Rotating} & \makecell{Linear\\ Control} & \makecell{Velocity\\ Damping} & \makecell{Configuration-\\Dependent Torque} & \makecell{Variable\\ Dynamics} \\
    \midrule
\parbox[t]{2mm}{\multirow{4}{*}{\rotatebox[origin=c]{0}{\textit{\shortstack{$t=0.8$ \\ \textit{seconds}}}}}}
 & LEAP (\cite{jhin2023learnable}) & $7.67 \pm 4.76$ & $3.98 \pm 3.76$ & $4.12 \pm 1.91$ & $5.45 \pm 2.58$ & $11.45 \pm 7.67$ \\
 & Conservational (\cite{mason2023learning}) & $2.63 \pm 0.22$ & $2.16 \pm 0.03$ & $2.14 \pm 0.02$ & $2.17 \pm 0.03$ & $2.72 \pm 0.24$ \\
 & \SO-GRU & \textcolor{blue}{$1.65 \pm 0.06$} & \textcolor{blue}{$0.82 \pm 0.03$} & \textcolor{blue}{$0.78 \pm 0.02$} & \textcolor{blue}{$0.88 \pm 0.02$} & \textcolor{blue}{$1.81 \pm 0.10$} \\
 & \SO-nCDE & $2.25 \pm 0.18$ & $1.11 \pm 0.07$ & $1.07 \pm 0.05$ & $1.23 \pm 0.07$ & $2.30 \pm 0.22$ \\
 & SG-nCDE (Ours) & \boldmath \textbf{$0.87 \pm 0.05$} & \boldmath \textbf{$0.49 \pm 0.03$} & \boldmath \textbf{$0.42 \pm 0.02$} & \boldmath \textbf{$0.58 \pm 0.03$} & \boldmath \textbf{$0.89 \pm 0.07$} \\
\midrule
\parbox[t]{2mm}{\multirow{4}{*}{\rotatebox[origin=c]{0}{\textit{\shortstack{$t=1.2$ \\ \textit{seconds}}}}}}
 & LEAP (\cite{jhin2023learnable}) & $11.78 \pm 1.56$ & $6.51 \pm 2.73$ & $4.76 \pm 0.79$ & $7.28 \pm 1.51$ & $15.12 \pm 5.95$ \\
 & Conservational (\cite{mason2023learning}) & $4.09 \pm 0.34$ & $3.18 \pm 0.05$ & $3.14 \pm 0.04$ & $3.20 \pm 0.05$ & $4.28 \pm 0.44$ \\
 & \SO-GRU & \textcolor{blue}{$2.67 \pm 0.11$} & \textcolor{blue}{$1.34 \pm 0.05$} & \textcolor{blue}{$1.12 \pm 0.04$} & \textcolor{blue}{$1.43 \pm 0.05$} & \textcolor{blue}{$2.71 \pm 0.19$} \\
 & \SO-nCDE & $3.77 \pm 0.26$ & $1.76 \pm 0.10$ & $1.68 \pm 0.09$ & $1.83 \pm 0.10$ & $3.89 \pm 0.42$ \\
 & SG-nCDE (Ours) & \boldmath \textbf{$1.28 \pm 0.08$} & \boldmath \textbf{$0.64 \pm 0.03$} & \boldmath \textbf{$0.50 \pm 0.03$} & \boldmath \textbf{$0.74 \pm 0.04$} & \boldmath \textbf{$1.31 \pm 0.14$} \\
    \bottomrule
  \end{tabular}%
  \vspace{-3mm}
  }
\caption{Trajectory prediction accuracy on $0.8$ and $1.2$ second extrapolation in rotational error (degrees). We evaluate several recent methods for forecasting temporal dynamics on scenarios encompassing known and unknown moments of inertia and various external forces. \textbf{Best} and \textcolor{blue}{second-best} results are emphasized. \vspace{-3mm}}
 \label{table:main_results}
\end{table*}

\subsection{Higher-order Control Paths}

Besides the encoded latent initial condition $z_{t_0}$ and the first-order derivative $\dX$, our model has no input on how angular velocity changes over time.
This can cause temporal phase shifts where trajectories with the correct predicted orientation lag behind ground truth trajectories.

\begin{lemma}[Second-Order Neural CDE]
\label{lemma:second_order_cde}
We propose a second-order Neural CDE that incorporates both the first and second derivatives of the control path:
\begin{equation}
\bz_t = \bz_{t_0} + \int_{t_0}^t \ftheta(\bz_s) \dX_s + \gtheta(\bz_s) \ddX_s
\end{equation}
for $t \in (t_0, t_n]$ where $\ftheta$ and $\gtheta$ are neural networks with parameters $\thetaparam$ and $\thetaparam'$, and $\bz_{t_0}$ is the initial condition. 
\end{lemma}

\noindent The second-order term $\gtheta(\bz_s) \ddX_s$ provides explicit acceleration information to the neural ODE, enabling more accurate modeling of systems with non-constant angular velocity.
The Savitzky-Golay filter on $\SO$ produces both $\dX$ and $\ddX$ in closed form, making it particularly well-suited for our proposed Neural CDE formulations.

\subsection{Learning Rotational Kinematics}

The literature supports using 9D rotation matrices as a latent representation for learning on \SO~\cite{zhou2019continuity,bregier2021deepregression,geist2024learning}.
We therefore use the 9D rotational derivatives $\dot{R} = R \hat{\bomega}$ obtained from \cref{sec:robust_regression} directly as an integration path for the latent state in \cref{eq:neural_cde}. 
Subsequently, the time-evolved hidden state of the Neural CDE is projected into the 6D representation \cite{zhou2019continuity}, which has been demonstrated suitable as a learning target of a neural network:
\begin{align}
    \label{eq:6D_representation}
    \mathbf{r_k}=\left(\mathbf{\nu}_1, \mathbf{\nu}_2\right) \in \mathbb{R}^{3 \times 2}
\end{align}

The prediction $\mathbf{y_k} \in \SO$ can then be recovered via Gram-Schmidt orthonormalization (GSO) \cite{geist2024learning}.
We now optimize to minimize the geodesic distance to learn the latent code $\theta$ of the network $\mathbf{f}_\theta$ in \cref{eq:neural_cde}:
\begin{equation}
\label{eq:general_loss}
\arg \min_\theta \sum_{k=1}^m ||{\logmap(\mathbf{y_k}}\mathbf{x_k^{-1}})||_F
\end{equation}

\noindent where $\mathbf{\hat{y}_k}$ is the rotation matrix reconstructed via GSO obtained by solving the Neural CDE at time $t_k$, $\mathbf{x_k}$ is the ground-truth observation, and $||\cdot||_F$ denotes the Frobenius norm.
During inference, we proceed by evolving the differential equation in \cref{eq:neural_cde} forward in time for each future point $k \in [t_{N+1}, t_{N+m}]$ (see \cref{fig:cde_main_figure}).

\section{Emperical Evaluation}
\label{sec:dataset_generation}

We conduct extensive experiments on simulated rotational trajectories to understand the behavior of rotation extrapolation methods.
Furthermore, we use models trained in simulation for rotational forecasting on real-world data from the Oxford Motion Dataset (MOD) \cite{judd2019oxford} and rotational estimates from a modern 6D-pose estimator \cite{wang2021gdr} to demonstrate generalization to unseen dynamics and sensor noise.

\paragraph{$\mathbf{SO}(3)$ Kinematics}
The motion of rotating rigid bodies follows well-established physical principles that can be derived from classical mechanics; using these, we can simulate diverse rotational trajectories to study \SO~forecasting in a principled manner.

\begin{definition}[Moment of Inertia Tensor]
\label{def:moi_tensor}
The moment of inertia (MOI) tensor $\mathbf{J} \in \mathbb{R}^{3 \times 3}$ is symmetric positive-definite characterizing a rigid body's resistance to rotational motion. There exists a coordinate system (principal axes) where $\mathbf{J}$ is diagonal \cite{goldstein2002classical} (see App. C.1).
\end{definition}

In real-world applications with dissipative forces and measurement noise, strict conservation laws may not always hold.
We, therefore, simulate both conservational (as in \cite{mason2023learning}) and non-conservational ($\tau_{\text{ext}} \neq 0$) trajectories.

\begin{definition}[Rotational Rigid Body Dynamics]
\label{def:rotational_dynamics}
The following equations govern rigid rotational dynamics in 3D:
\begin{align}
\dot{R} &= R \hat{\omega}, \\
\dot{\omega} &= \bJ^{-1} (\tau_{\text{ext}} - \omega \times \bJ \omega),
\end{align}
where:
\begin{itemize}
\item $R(t) \in \SO$ is the rotation matrix representing orientation,
\item $\omega (t) \in \mathbb{R}^3$ is the body-frame angular velocity,
\item $\mathbf{J} \in \mathrm{R}^{3 \times 3}$ is the (diagonal) inertia matrix,
\item $\tau_{\text{ext}} \in \mathbb{R}^3$ represents external torques.
\end{itemize}
\end{definition}

\paragraph{Experimental Design} 
We evaluate our method on five simulated scenarios representing diverse, realistic dynamics by integrating the equations in \cref{def:rotational_dynamics} forward in time with $dt = 10^{-3}$. 
To understand the behavior of non-conservative systems, we explicitly define different representations for external forces $\tau_{\text{ext}}$:

\begin{itemize}
    \item \textit{Freely Rotating Body:} No external torques ($\tau_{\text{ext}} = 0$). Angular momentum is conserved.
    \item \textit{Linear Torque Control:} A linear control law $\tau_{\text{ext}} = \mathbf{J}(\bA\omega + \bb)$ where $\bA \in \mathbb{R}^{3\times3}$.
    \item \textit{Velocity Damping:} Direct damping as a function of angular velocity through $\tau_{\text{ext}} = \mathbf{J}(\bD\omega)$ with negative definite damping matrix $\bD$.
    \item \textit{Combined External Torque:} Superposition of damped rotation and an external torque computed in the world frame, representing conservative force fields (e.g., magnetic, gravitational, or torsional forces) where the torque depends explicitly on orientation with contributions weighted by $\bJ$.
    \item \textit{Variable Dynamics:} Trajectories sampled uniformly from the four scenarios above.
\end{itemize}

\paragraph{Data Simulation} 
We generate $8000$ 10-second trajectories for each simulation above, separated into four splits.
Each split is simulated based on a distinct MOI distribution to assess model generalization.
We use 2/1/1 splits for training, validation (model selection and hyperparameter tuning), and testing.
During each training iteration, a model sees a $1.2$ second trajectory segment sampled uniformly at $dt=10^{-1}$ and is regularized on predictions for the next $0.8$ seconds.
Batches are sampled randomly from all trajectories and their segments within a given split.
Additional details are provided in \App B. 1.

\paragraph{Evaluation Metrics}
During evaluation, we measure the rotational geodesic error (RGE) \cite{huynh2009metrics} each prediction:
\begin{equation}
\mathrm{RGE}(R_1, R_2) = 2 \arcsin\left( \frac{\| R_2 - R_1 \|_F}{2\sqrt{2}} \right)    
\end{equation}

\noindent Noise for the input sequences is sampled from $\mathcal{N}(0, \delta)$ with $\delta \in [0.01, 0.05]\pi$. $0.05 \pi$ corresponds to a rotation of $\approx 9$°, which is already considerable noise (for instance, in 6D pose estimation $2^\circ$ is used as a threshold \cite{wang2021gdr}).
Results for $\delta = 0.05\pi$ are reported in \cref{table:main_results}.

\begin{figure*}[t!]
    \vspace{-3mm}
    \includegraphics[width=\textwidth]{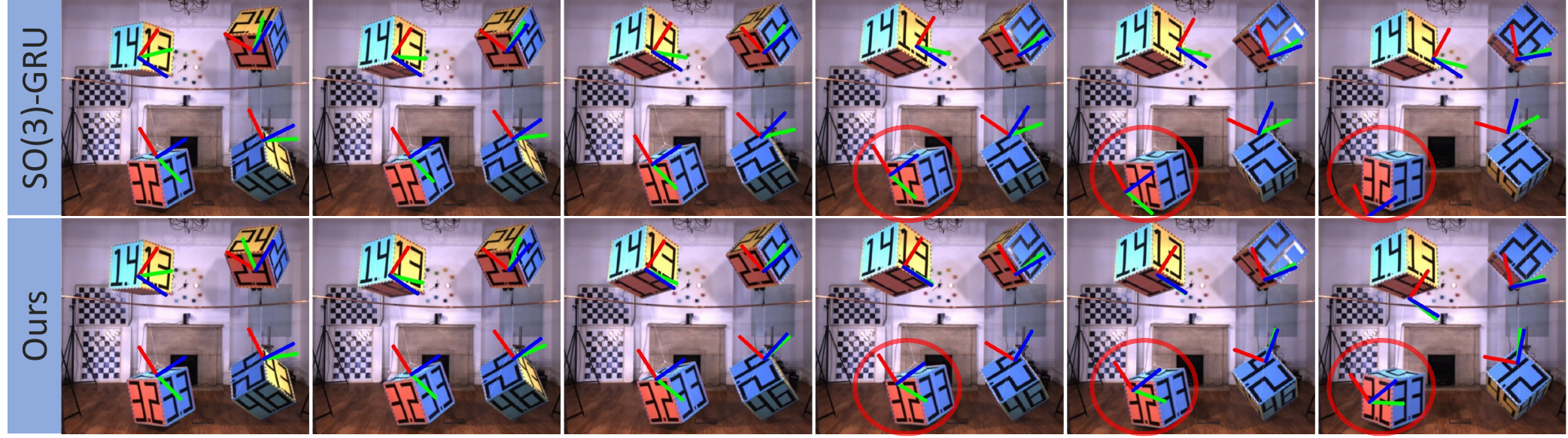}
    \caption{Qualitative comparison of the extrapolation results (left to right) produced by our method (top row) and the \SO-GRU baseline (bottom row) on the OMD dataset. Methods are conditioned on the preceding timesteps (not visualized). Rotation is represented using 3D axis coordinates. 
    Ground truth observations relative to the world coordinate system are used for the translation component. \vspace{-5.0mm}}
    \label{fig:OMD_qualitative}
\end{figure*}

\paragraph{Baselines}
We compare our method with the Conservation of Energy approach for \SO~prediction \cite{mason2023learning}, \SO-GRU and spline \SO-nCDE baselines we introduce, and the learnable path CDE (LEAP) \cite{jhin2023learnable} on the 10 scenarios described in \cref{sec:dataset_generation}.
As the conservation of energy method~\cite{mason2023learning} solves an initial value problem, we provide the additionally required angular momentum estimate to integrate forward in time from the $\tilde{\mathbf{x}}_n$.
Momentum estimates $\tilde{\mathbf{L}}$ are sampled from a noise distribution proportional to the noisy state estimates $\tilde{\mathbf{x}}$.
Details for baseline models can be found in \App B.2.

\paragraph{Implementation Details}
We use Python version 3.12, PyTorch \cite{paszke2019pytorch} for data simulation, and diffrax \cite{kidger2021on} for numerical integration of our method using Dormand-Prince 4/5 to integrate the dynamical systems.
Experiments are run on an RTX 5000 GPU.
Adam is employed as an optimizer with a learning rate of $5 \times 10^{-3}$.

\paragraph{$\mathbf{SO(3)}$ Trajectory Visualizations}
Quaternions offer a convenient visualization method as they can be reduced to three dimensions while preserving both the angle and axis of rotation.
Let $\mathbf{q} = [cos(\theta/2), sin(\theta/2)\mathbf{u}]$ be a quaternion where $\theta$ is the rotation angle, and $\mathbf{u}$ is the unit vector representing the axis of rotation (the axis-angle representation).
By normalizing the last three coordinates, we project trajectories onto the unit sphere $\mathbb{S}^2$.
Visualizations of longer forecasts of 1.2 seconds (during training, the model only sees 0.8s) are depicted on $\mathbb{S}^2$ in \cref{fig:trajectory} for the \textit{velocity damping} scenario.

\subsection{Analysis \& Applications}
\label{sec:experimental_main}
The results for simulated experiments are depicted in \cref{table:main_results}.
LEAP \cite{jhin2023learnable} generally does not successfully reconstruct or forecast \SO~trajectories.
Nevertheless, this motivates using a robust filtering approach for the CDE instead of learning the path in an unstructured way.
On the other hand, the baseline \SO~spline neural CDE predicts plausible trajectories, particularly within the training horizon of $t=0.8s$.
The \SO-GRU performs competitively in several cases, particularly with dissipative forces.
Our proposed method tends to perform the most consistently.

We observe the GRU and our method predictions are competitive beyond the training horizon (see \cref{fig:trajectory}). 
At the same time, the \SO-nCDE baseline and LEAP tend to diverge sharply after several predicted timesteps.
The conservation approach \cite{mason2023learning} predicts plausible trajectories, yet visualized trajectories in non-conservational cases confirm that external forces are not accounted for.
In \App B.4 we show that even with ground truth momentum estimates, the trajectores diverge for non-conservational cases.

\begin{table}[t]
\resizebox{0.99\columnwidth}{!}{
   \setlength{\tabcolsep}{7pt}
   \small
   \centering
  \begin{tabular}{l|ccc}
    \toprule
    Method & \makecell{Translation\\ Motion} & \makecell{Unconstrained\\ Motion} & \makecell{Static\\ Motion} \\
    \midrule
    \SO-nCDE & $6.49 \pm 9.42$ & $5.43 \pm 7.09$ & $7.82 \pm 12.19$ \\
    \SO-GRU & \textcolor{blue}{$3.04 \pm 1.29$} & \textcolor{blue}{$2.90 \pm 1.25$} & \textcolor{blue}{$2.83 \pm 1.26$} \\
    SG-nCDE (Ours) & \boldmath \textbf{$2.32 \pm 1.03$} & \boldmath \textbf{$2.30 \pm 1.02$} & \boldmath \textbf{$2.18 \pm 1.00$} \\
    \bottomrule
  \end{tabular}%
  \vspace{-7mm}
  }
  \caption{Trajectory prediction error (degrees) across different motions in the OMD dataset \cite{judd2019oxford}. We evaluate methods in 3 different scenarios (translational, unconstrained, and static). \textbf{Best} and \textcolor{blue}{second-best} results are emphasized. \vspace{-3mm}}
  \label{table:omd_dataset}
\end{table}

\paragraph{Sim2Real: Unseen Dynamics and Noise}
We evaluate the capabilities of models trained on the \textit{Variable Dynamics} scenario in \cref{sec:experimental_main} on trajectories in the Oxford Motion Dataset (OMD)~\cite{judd2019oxford} consisting of observations of four swinging boxes for three different camera motions (translational, unconstrained, and static).
Vicon tracking data is downsampled to 40Hz, and models are evaluated on predicting 0.3 seconds.
Statistics for results aggregated over the four boxes in each scene are depicted in \cref{table:omd_dataset}.
We observe that SG-nCDEs consistently outperform despite the diverse motions present. 

We additionally present a qualitative comparison of extrapolation performance in \cref{fig:OMD_qualitative}, highlighting the differences in extrapolation quality between our method and the \SO-GRU baseline.
As we represent translation with respect to the world coordinate system, errors accrued in rotation estimation manifest as offsets in pixel space.
This is apparent later in the image sequence, where the predicted rotations from the \SO-GRU baseline cause the projected coordinate system to drift significantly. 
The boxes have different translational velocities, with the two right-most boxes (\#2 and \#4) yielding the most stable overall predictions.

\begin{figure*}[t!]
    \vspace{-3mm}
    \includegraphics[width=\textwidth]{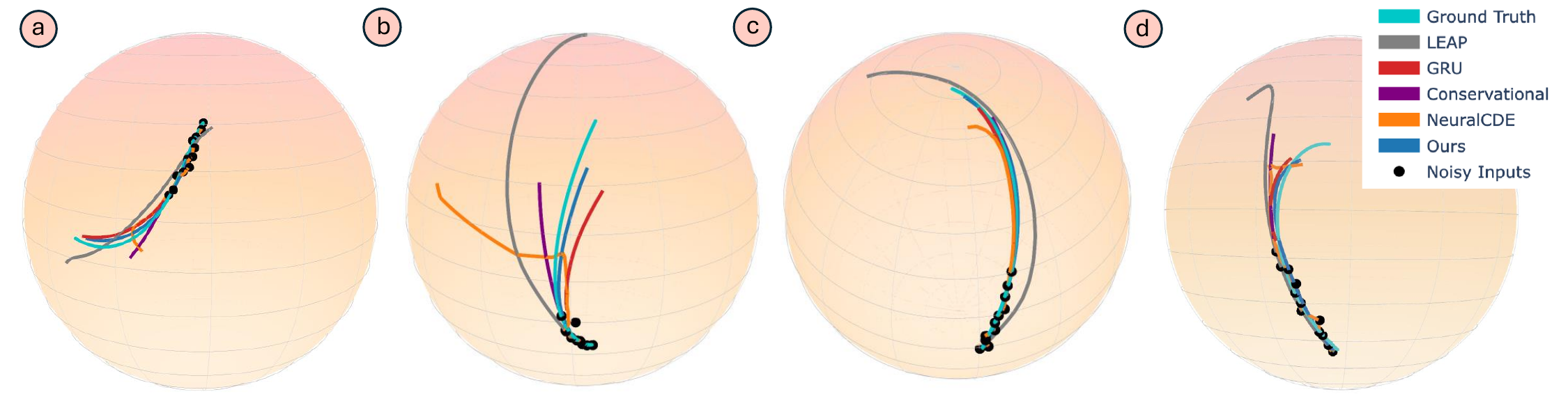}
    \caption{Trajectory extrapolations visualized on the unit sphere $\mathbb{S}^2$ for the proposed method (blue) and the relevant baselines on the \textit{damped} scenario. We additionally depict the ground truth (teal) and the noisy state estimates (black) provided as input for reference. Trajectories are visualized for $t=1.2s$ future timesteps. \vspace{-3mm}}
    \label{fig:trajectory}
\end{figure*}

\paragraph{Dynamic Pose Estimation}
To gauge the performance in conjunction with 6D object pose estimators, we apply our method to the outputs of the widely used GDR-NPP~\cite{wang2021gdr,liu2022gdrnpp_bop}.
We fine-tune the network weights published by the authors on a single object and apply the pose estimation to a sequence of (unseen) simulated images depicted in~\cref{fig:object_pose}, with detailed comparisons in \App A.1.
In this way, dynamics estimation can be decoupled from object pose estimation, mitigating the cost of annotating dynamic objects.

\paragraph{Ablation Study}
We analyze the performance of the baseline \SO-nCDE with our proposed modifications in \cref{table:ablation} in the \textit{freely rotating} scenario.
This includes Savitzky-Golay filtering as a control signal for the CDE, iterative denoising, the second order control signal, and learning the weight matrix $\mathbf{W}$, which all notably increase performance over the baseline model.
Notably, a naive adaptation of the Savitzky-Golay filter does not initially improve performance due to boundary artifacts \cite{schmid2022and}; the strong performance of our method is a culmination of the possibilities opened up by using robust control signals together with neural CDEs.

\paragraph{Number of Function Evaluations (NFE)}
NFE assesses computational efficiency in differential equation solvers correlating with the complexity of the numerical integration. 
We observe that SG-nCDEs have lower and more consistent NFE counts throughout the training process; our 2nd order model requires around half as many function evaluations, while the first-order filter model requires more than an order of magnitude less than the baseline \SO-nCDE. 
This implies the that Savitzky-Golay filtering yields a more easily numerically integrable function in \cref{def:neural_cde}.

\begin{figure}[t]
    \vspace{-3mm}
    \centering
    \includegraphics[width=0.95\columnwidth]{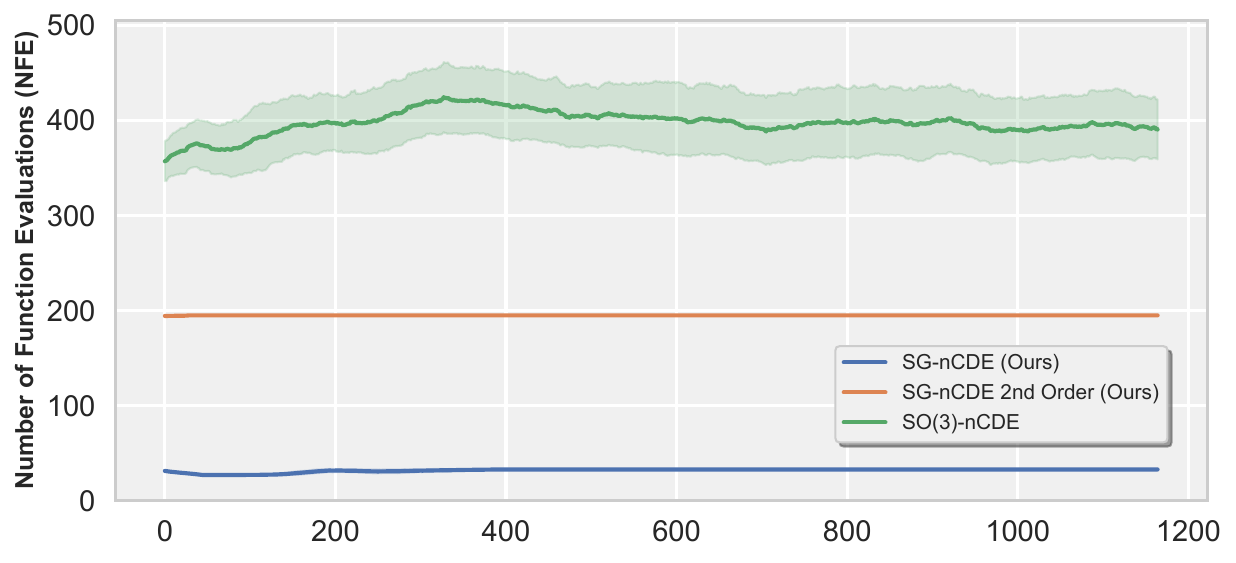}
    \vspace{-1mm}
    \caption{Comparison of the number of function evaluations (NFE) per training step between the baseline model and our model. The shaded areas represent one standard deviation.}
    \label{fig:nfe_comparison_plot}
    \vspace{-6mm}
\end{figure}

\begin{table}[t]
\vspace{-1.5mm}
    \resizebox{0.47\textwidth}{!}{
    \centering
    \setlength{\tabcolsep}{4pt}  %
    \begin{tabular}{l|ccc|c|c}
    \toprule
    & \makecell[c]{Robust SG \\ Regression}
    & \makecell[c]{2nd-Order \\ Control Signal}
    & \makecell[c]{Adaptive \\ $\mathbf{W}$}
    & \makecell[c]{Freely Rotating \\ (RGE)}
    & \makecell[c]{Runtime \\ (ms)} \\
    \midrule
    i &  &  &  & $3.62\pm0.24$ & $74.6 \pm 0.7$\\
    ii & \checkmark &  &  & $1.91\pm0.12$ & $60.9 \pm 1.4$\\
    iii & \checkmark &  & \checkmark & $1.98\pm0.13$ & $60.9 \pm 1.2$\\
    iv & \checkmark & \checkmark &  & $1.31\pm0.08$ & $74.5 \pm 1.1$\\
    v & \checkmark & \checkmark & \checkmark & $1.28\pm0.08$ & $76.8 \pm 1.2$\\
    \bottomrule
    \end{tabular}
    }
    \caption{Ablating our modeling components and runtime. We study the robust Savitzky-Golay (SG) regression, adaptive weighting $\mathbf{W}$, and the 2nd-order CDE. (i) represents the \SO-nCDE baseline with no other modifications. Extrapolation is evaluated in rotational geodesic error at 1.2s on the freely rotating scenario.}
    \label{table:ablation}
    \vspace{-4mm}
\end{table}

\paragraph{Runtime}
As evident in \cref{table:ablation}, SG-nCDEs introduce marginal additional computational overhead. 
While these techniques improve accuracy, practitioners should weigh the performance tradeoffs against their application requirements.

\vspace{-1mm}\section{Conclusion}
We propose a method for extrapolating rotational motion for dynamical systems in \SO~with unknown physical parameters.
This problem is challenging and understudied; previous works either consider pairwise predictions or assume energy conservation, neglecting long-term forecasting in \SO.
Our method achieves more accurate predictions than existing works by regressing a robust control path directly on the manifold of rotations \SO.
This path is used as a control signal for a neural CDE modeling the evolving dynamical system.
Our approach outperforms existing methods in accuracy and extrapolation robustness across several settings.

\paragraph{Limitations and future work}
Our approach is limited to the special orthogonal group \SO, whereas the group of rigid motions \SE~also involves translational components. 
As the extrapolation of Euclidean quantities has been extensively studied, extending our framework to SE(3) presents an interesting opportunity for future work.
Particularly exciting is the study of coupled motion between rotation and translation (e.g., centripetal and Coriolis forces) and filtering in the Lie algebra of twists.
Nevertheless, robust filtering on \SE~risks inadvertently decoupling rotation and translation \cite{sommer2020efficient,wang2008computation,mankovich2023chordal}, posing problems when studying gyroscopic coupling effects.
In the future, we will also explore different path parameterizations.

\clearpage

\paragraph{Acknowledgements}
The authors would like to thank Junwen Huang and Benjamin Busam for fruitful early discussions.
The authors are grateful for the support of the Excellence Strategy of German local and state governments, including computational resources of the LRZ AI service infrastructure provided by the Leibniz Supercomputing Center (LRZ), the German Federal Ministry of Education and Research (BMBF), and the Bavarian State Ministry of Science and the Arts (StMWK).
T. Birdal acknowledges support from the Engineering and Physical Sciences Research Council [grant EP/X011364/1]. 
T. Birdal was supported by a UKRI Future Leaders Fellowship [grant MR/Y018818/1].

{\small
\bibliographystyle{ieeenat_fullname}
\bibliography{references}
}

\clearpage
\onecolumn

\begin{center}
  \Large\bf \paperTitle\\
  \vspace{0.5em}
  \Large\bf Appendix
  \vskip 1.5em

  \begin{tabular}{c}
    \large
    Lennart Bastian\textsuperscript{1,2*} \quad
    Mohammad Rashed\textsuperscript{1,2*} \quad
    Nassir Navab\textsuperscript{1,2} \quad
    Tolga Birdal\textsuperscript{3} \\
    \normalsize
    \textsuperscript{1} Technical University of Munich \;
    \textsuperscript{2} Munich Center of Machine Learning \;
    \textsuperscript{3} Imperial College London \\
  \end{tabular}
  \vskip 2em
\end{center}

\renewcommand{\contentsname}{Appendix Contents}
\begin{center}
{\Large \bf Appendix Contents}
\end{center}
\vskip 1em

{\small
\setlength{\parindent}{0pt}
\setlength{\parskip}{3pt}

\textbf{\hyperref[sec:appendix-experiments]{A \quad Experiments and Ablations}} \dotfill \pageref{sec:appendix-experiments}

\hspace{0.5cm} \hyperref[subsec:6d-pose-estimation]{A.1 \quad 6D Pose Estimation} \dotfill \pageref{subsec:6d-pose-estimation}

\hspace{0.5cm} \hyperref[subsec:sensor-fusion]{A.2 \quad Application: Irregularly Sampled Sensor Fusion} \dotfill \pageref{subsec:sensor-fusion}

\hspace{0.5cm} \hyperref[subsec:hyperparams]{A.3 \quad Hyperparameters} \dotfill \pageref{subsec:hyperparams}

\hspace{0.5cm} \hyperref[subsec:integration-path]{A.4 \quad On the Choice of Integration Path} \dotfill \pageref{subsec:integration-path}

\hspace{0.5cm} \hyperref[subsec:noise-performance]{A.5 \quad Model performance vs Input Noise} \dotfill \pageref{subsec:noise-performance}

\textbf{\hyperref[sec:appendix-implementation]{B \quad Implementation Details}} \dotfill \pageref{sec:appendix-implementation}

\hspace{0.5cm} \hyperref[subsec:experimental-design]{B.1 \quad Experimental Design} \dotfill \pageref{subsec:experimental-design}

\hspace{0.5cm} \hyperref[subsec:inertia-distributions]{B.2 \quad Moment of Inertia Distributions for Simulation} \dotfill \pageref{subsec:inertia-distributions}

\hspace{0.5cm} \hyperref[subsec:simulation-scenarios]{B.3 \quad Simulation Scenarios} \dotfill \pageref{subsec:simulation-scenarios}

\hspace{0.5cm} \hyperref[subsec:baseline-design]{B.4 \quad Baseline Model Design} \dotfill \pageref{subsec:baseline-design}

\textbf{\hyperref[sec:appendix-background]{C \quad Additional Background}} \dotfill \pageref{sec:appendix-background}

\hspace{0.5cm} \hyperref[subsec:inertia-diagonalization]{C.1 \quad Moment of Inertia and Diagonalization} \dotfill \pageref{subsec:inertia-diagonalization}

\hspace{0.5cm} \hyperref[subsec:exp-log-maps]{C.2 \quad Exponential and Logarithmic Maps} \dotfill \pageref{subsec:exp-log-maps}

\hspace{0.5cm} \hyperref[subsec:savitzky-golay]{C.3 \quad SO(3) Savitzky-Golay Filtering} \dotfill \pageref{subsec:savitzky-golay}

\textbf{\hyperref[sec:appendix-proof]{D \quad Proof of Proposition 10}} \dotfill \pageref{sec:appendix-proof}

}

\vskip 2em
\setcounter{footnote}{0}%

\appendix
\clearpage
\section{Experiments and Ablations}\label{sec:appendix-experiments}

\begin{figure*}[t!]
    \includegraphics[width=\textwidth]{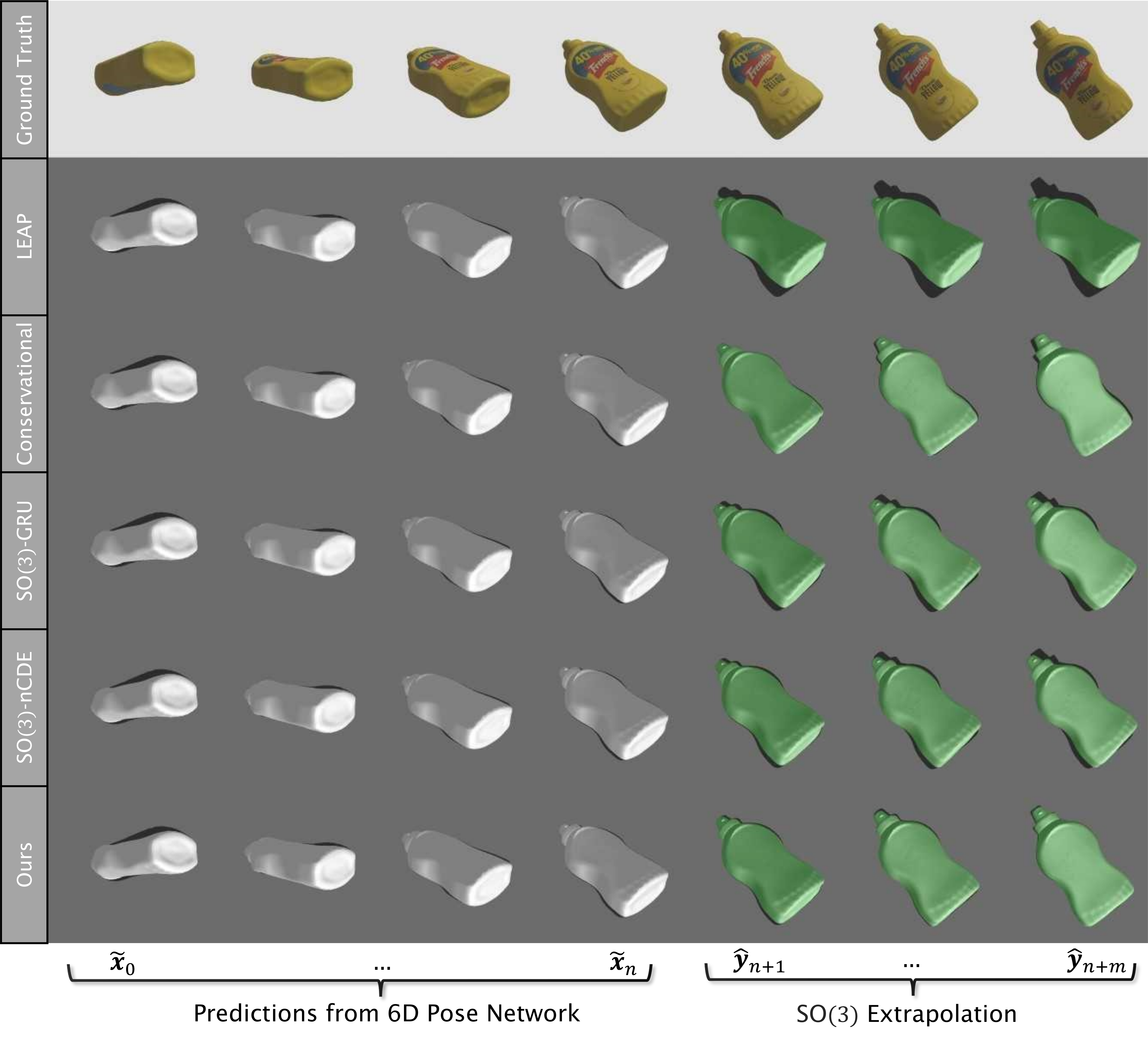}
    \caption{Application: 6D object pose estimation. We compare our method with the baselines by applying all models on the pose prediction outputs of GDRNet~\cite{wang2021gdr}. 6D pose predictions (bottom left) are used as input to the \SO extrapolation methods. Extrapolated outputs are depicted in green. Our method is able to predict more robust future states based on noisy pose estimates.}
    \label{fig:object-pose-comparison-all-models}
\end{figure*}

\subsection{6D Pose Estimation}\label{subsec:6d-pose-estimation}
In this section, we provide additional details regarding the 6D Pose estimation application introduced in the main paper. 
A total of 25k images are generated using BlenderProc2~\cite{Denninger2023}. 
The synthetic scenes featured a single object (a mustard bottle) from the YCB-Video dataset~\cite{xiang2018posecnn}, with the object's orientation randomly sampled. 
We then train GDRNet ~\cite{wang2021gdr} on individual (non-sequential) RGB images. 
The publicly available pre-trained weights are fine-tuned for one training epoch, without data augmentations.

Figure \ref{fig:object-pose-comparison-all-models} illustrates simulated trajectories of the object for the configuration-dependent torque scenario. 
In this experiment, image sequences are given as input to the pose estimator, producing noisy \SO~rotation trajectories.
These are then used as input for the \SO~extrapolation methods.
We compare our method to the baselines LEAP, Conservational, \SO-GRU, and \SO-nCDE. 

As shown in Figure \ref{fig:object-pose-comparison-all-models}, the LEAP baseline struggles to handle the rotational behavior effectively. 
The conservational approach effectively predicts rotation within this time horizon but does not capture the damping behavior and overshoots slightly. 
While the \SO-GRU and \SO-nCDE baselines perform better, managing to account for the damped patterns, the extrapolated poses produced by these methods exhibited higher errors than our approach.

\subsection{Application: Irregularly Sampled Sensor Fusion}\label{subsec:sensor-fusion}

\begin{wrapfigure}{r}{0.375\textwidth}
    \centering
    \vspace{-5.0mm}
    \includegraphics[width=\linewidth]{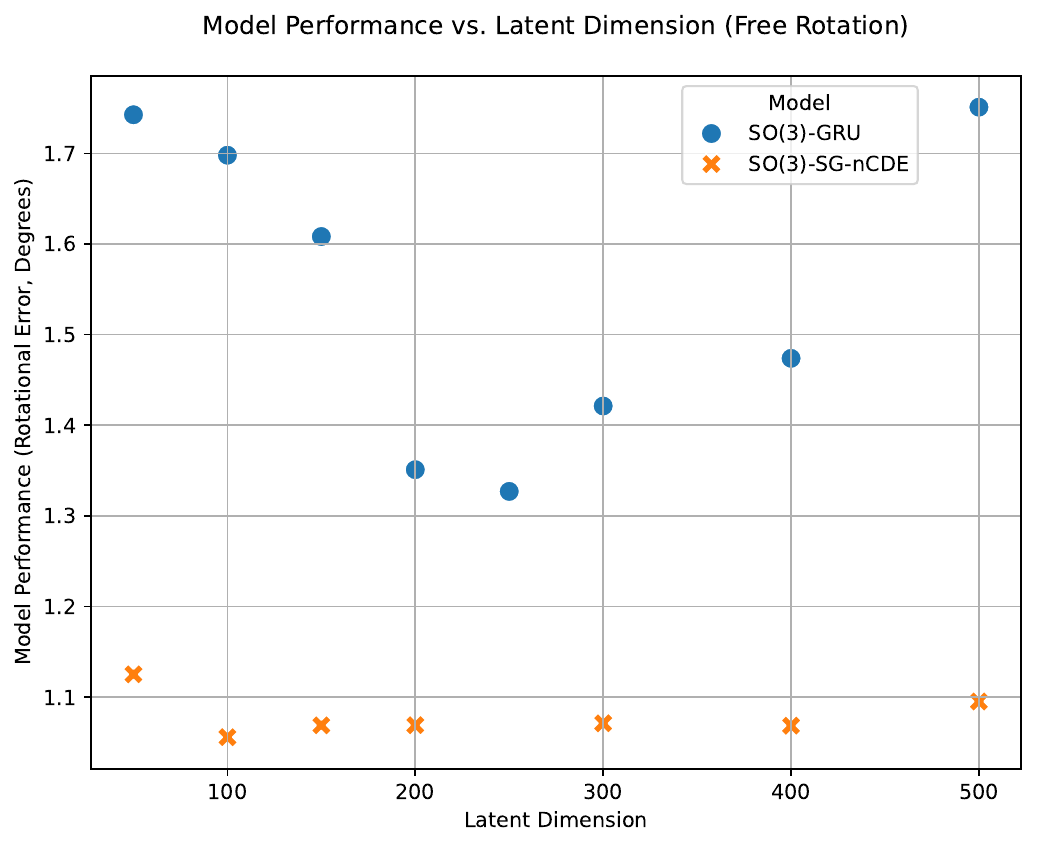}
    \caption{Ablating latent dimension vs. prediction performance in GRU vs. neural CDEs on the \textit{freely rotating} scenario with $\delta = 0.05$ and a prediction horizon of $t=0.8s$.}
    \label{fig:ablation_hyperparams}
    \vspace{-5.0mm}
\end{wrapfigure}

To evaluate the capabilities of the proposed method in real-world scenarios, we study the case of extrapolation in a sensor fusion application.
The data rig consists of a tablet depicting an ArUco \cite{garrido2014automatic} marker that streams the onboard inertial measurement unit (IMU) sampled at 100Hz sensor signal to a system that simultaneously captures the 6D pose of the tablet with a camera (sampled at 30Hz) as shown in figure \cref{fig:sensor_fusion}. 
The sensor is then moved manually and tracked, a particularly challenging (and potentially ill-defined) scenario as stochastic external torques can influence the device during the extrapolation period.
Sensors are calibrated spatiotemporally using a Levenberg-Marquardt optimization, after which we perform outlier rejection based on a simple median-absolute-deviation heuristic.
We then capture 24 trajectories with the sensor rig to evaluate the models trained on the simulated \textit{combined external torque} scenarios.
The trajectories from the two sensors are naively merged according to the registered timestamps and then used as input to the models.

\begin{figure*}[t]
    \centering
    \includegraphics[width=\textwidth]{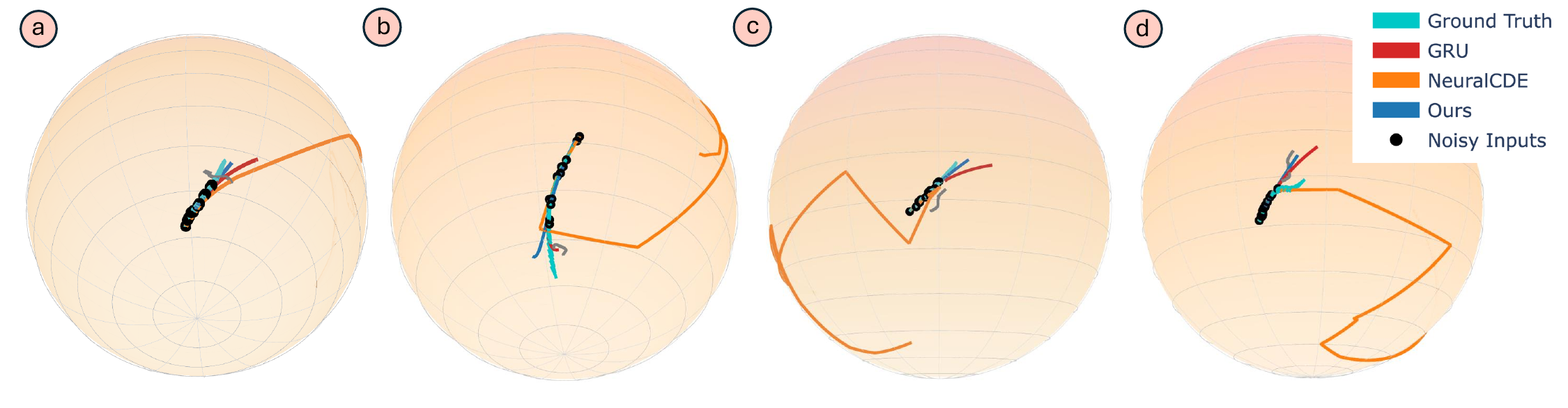}
    \caption{We evaluate the best-performing models' extrapolation capabilities under irregular sampling from noisy samples fused from multiple sensors. The trajectories are visualized on the unit sphere $\mathbb{S}^2$.}
    \label{fig:irregular_sampling}
    
    \vspace{1em} %
    
    \begin{tabular}{@{}p{0.58\textwidth}@{\hspace{0.04\textwidth}}p{0.38\textwidth}@{}}
    
    \begin{minipage}{\linewidth}
        \includegraphics[width=\linewidth]{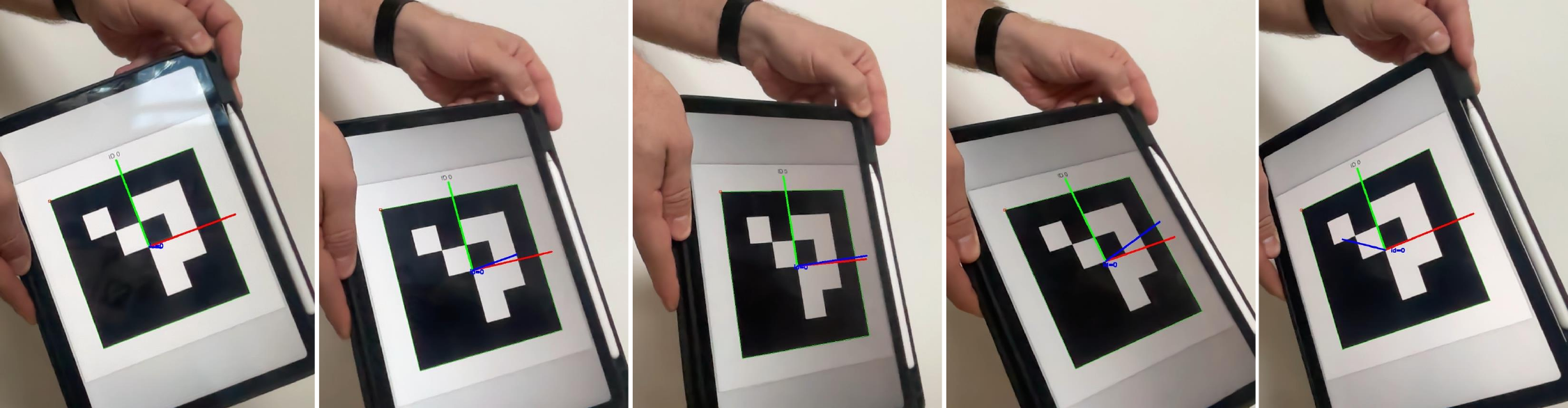}
        \caption{Evaluating the robustness of the proposed extrapolation methods for irregularly sampled extrapolation in the wild. The IMU signal from a tablet is synchronized with an external camera which captures the pose of a displayed ArUco marker. The resulting noisy and irregular trajectories are used to evaluate the models' extrapolation capabilities.}
        \label{fig:sensor_fusion}
    \end{minipage}
    
    & %
    
    \begin{minipage}{\linewidth}
        \captionof{table}{Comparison of prediction errors across different models from irregularly sampled noisy sensor measurements. Results show mean ± standard deviation in rotational geodesic error (RGE).}
        \label{table:irregular_sampling}
        \centering
        \vspace{1em} %
        \setlength{\tabcolsep}{4pt}
        \resizebox{0.9\columnwidth}{!}{
        \begin{tabular}{l|c}
        \toprule
        Model & Prediction Error \\
        \midrule
        LEAP & $18.30 \pm 5.45$ \\
        \SO-GRU & $11.39 \pm 0.73$ \\
        \SO-nCDE & $75.48 \pm 9.39$ \\
        Ours & $\mathbf{8.11 \pm 1.59}$ \\
        \bottomrule
        \end{tabular}
        }
    \end{minipage} \\
    
    \end{tabular}
\end{figure*}

\begin{figure}[t]
    \includegraphics[width=0.99\textwidth]{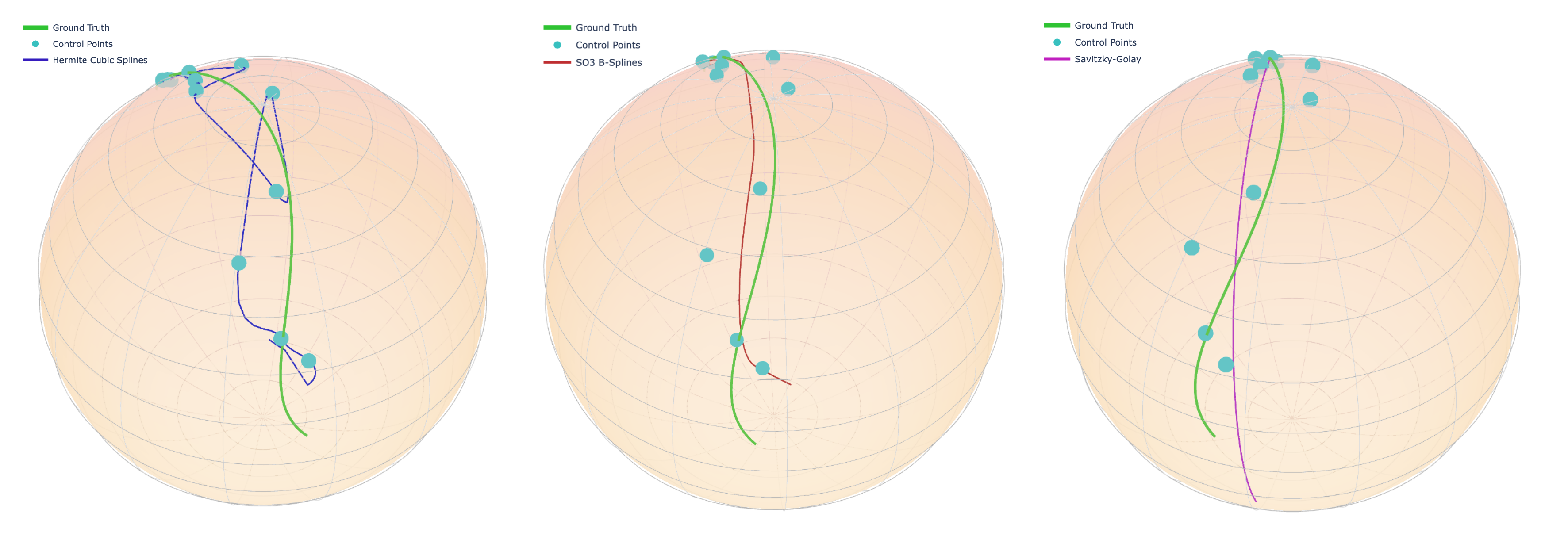}
    \vspace{-2mm}
    \caption{Comparison of interpolation methods on a particularly complex and noisy trajectory from the \textit{combined external force} scenario. From left to right: Hermite cubic splines with backward differences \cite{morrill2022choice}, \SO~B-splines \cite{sommer2020efficient}, and \SO~Savitzky Golay Filtering \cite{jongeneel2022geometric}.
    Trajectories represent interpolation and $t=0.2s$ extrapolation into the future.
    }
    \label{fig:spline_comparison}
\end{figure}

\noindent\textbf{Results} The results are depicted in \cref{table:irregular_sampling}, with visualizations in \cref{fig:irregular_sampling}.
Our method outperforms the baselines LEAP and \SO-nCDE by a large margin. 
In contrast to the simulated scenarios, the latter cannot accommodate the irregular, noisy signals and creates errant extrapolations.
The GRU performs competitively, despite the irregular sampling, and produces reasonable estimates.
\cref{fig:irregular_sampling} (d) depicts an example where the applied torque changes in a stochastic manner during the extrapolation horizon, which the models are not able to capture as they have no knowledge of such events.

\subsection{Hyperparameters}\label{subsec:hyperparams}
As the latent dimension significantly impacts model performance for both the GRU and nCDE-based methods, we provide an additional ablation over the latent dimension (see \cref{fig:ablation_hyperparams}).

While increasing the latent dimension of the GRU leads to overfitting, \SO-SG-nCDEs are robust to changes in latent dimension and overfitting.
\SO-nCDEs and \SO-SG-nCDEs are trained with a latent dimension of 125, while GRUs are trained with a latent dimension of 250 and 3 hidden units.
LEAP uses a latent dimension of 50; increasing the latent dimension for this method significantly increases runtime and VRAM but not performance.

\subsection{On the Choice of Integration Path}\label{subsec:integration-path}

Various control signals for neural CDEs have been proposed; most interpolation methods use cubic or higher order splines~\cite{kidger2020neural,morrill2022choice}. 
Kidger et al.~\cite{kidger2020neural} construct the integration path $\Xpath \in \mathcal{C}^2$ as a natural cubic spline over the input sequence.
While this guarantees smoothness of the first derivative, Hermite cubic splines with backward differences are preferred for CDE integration due to local control \cite{morrill2022choice}.
We observe that these choices do not obey the geometric structure of \SO; furthermore, they are not robust with respect to sensor noise interpolating the points erratically (see \cref{fig:spline_comparison}, left).
In general, such higher-order polynomials tend to diverge near the interpolation endpoints and must be extrapolated with some assumption, for instance, constant velocity \cite{persson2021practical}.
However, when dynamics are complicated, this assumption is overly simplifying, even more so than assuming conservation of energy as in \cite{mason2023learning}.

Efficient derivative calculations have been proposed for robust \SO~B-spline interpolation of trajectories in \SO~\cite{sommer2020efficient}.
These can be made suitable for extrapolation with, e.g., constant velocity extrapolation at the endpoint by constructing a nonlinear optimization to minimize residuals to the measurements, with additional smoothness near the endpoint.
We implement this in Ceres~\cite{Agarwal_Ceres_Solver_2022}, Basalt~\cite{usenko2018double}, and Sophus~\cite{sophus} to integrate Lie group constraints into the manifold optimization.
We use the efficient derivative computation of \cite{sommer2020efficient} and an additional first-order Taylor expansion using the computed Jacobians to include a constant velocity endpoint constraint defining the splines beyond their support interval.
While the splines nicely interpolate the noisy trajectory and can be used to extrapolate (see \cref{fig:spline_comparison}, center), they still exhibit bias towards endpoints.
Moreover, this approach is too computationally expensive to use during network training.
On a 24-Core Intel(R) Xeon(R) CPU, a single batch of 1000 trajectories takes $\approx 22$ seconds when maximally parallelized, a fraction of the trajectories we use for training.

On the other hand, the Savitzky-Golay filter provides a reasonable approximation at a low cost. 
As it merely requires solving a modest linear system, this means it can be differentiated through to learn a robust weighting suitable for extrapolation (see \cref{fig:spline_comparison}, right).

\subsection{Model performance vs Input Noise}\label{subsec:noise-performance}

In \cref{fig:noise_ablation}, we further evaluate the quality of the predictions over varying time horizons (1, 4, 6, 8) and level of noise $[0.02, 0.03, 0.04, 0.05]\pi$ on the \textit{velocity damping} scenario.
The results are depicted in \cref{fig:noise_ablation}.
Notably, all methods suffer from increased prediction errors as noise increases, but our method consistently outperforms the baselines.

\begin{figure}[t]
    \includegraphics[width=0.99\textwidth]{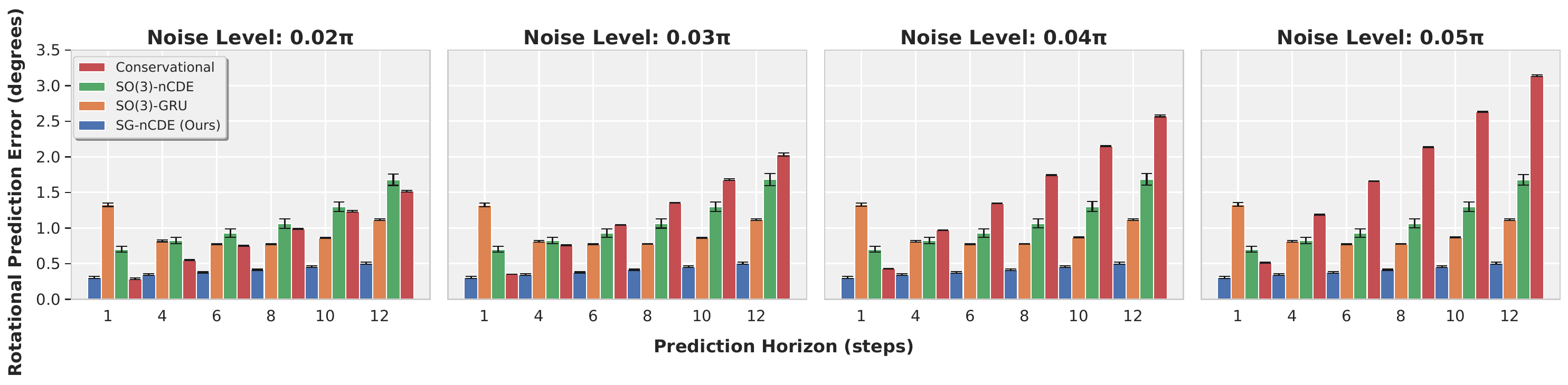}
    \vspace{-2mm}
    \caption{Rotational geodesic error for different input noise levels on the velocity damping scenario. We compare the performance of the baseline neural CDE and our best performing method (weighted, 2nd-order) with different amounts of simulated noise added to the simulated dynamical trajectories. Noise is sampled from $\mathcal{N}(0, \delta)$ with $\delta \in [0.02, 0.05]\pi$.
    }
    \label{fig:noise_ablation}
\end{figure}

\section{Implementation Details}\label{sec:appendix-implementation}

We provide further details regarding the five simulated experimental scenarios described in the main paper.

\subsection{Experimental Design}\label{subsec:experimental-design}
\label{appendix:dataset_simulation}

For each experiment, we sample an initial orientation according to the strategy in \cite{shoemake1992uniform}, and an initial angular velocity is sampled according to a truncated normal distribution to avoid degenerate cases with a small angular velocity. 

We sample via rejection sampling:
\begin{align*}
& \quad \text{1. Sample } X \sim \mathcal{N}(0, \sigma) \\
& \quad \text{2. Accept if } |X| > \eta, \text{ otherwise repeat step 1},
\end{align*}

\noindent with $\sigma = 0.3$ and $\eta = 0.1$

\subsection{Moment of Inertia Distributions for Simulation}\label{subsec:inertia-distributions}
\label{sec:inertia_distribution}

To evaluate model generalization across different rigid body properties, we simulate trajectories using four distinct moment of inertia (MOI) distributions. Each distribution is characterized by a base diagonal MOI tensor and additive noise:

\begin{equation}
\mathbf{J} = \text{diag}(\mathbf{J}_{\text{base}}) + \boldsymbol{\epsilon}, \quad \boldsymbol{\epsilon} \sim \mathcal{N}(0, \sigma^2\mathbf{I}_3)
\end{equation}

where $\mathbf{J}_{\text{base}}$ is the base MOI vector and $\sigma = 0.2$ is the noise standard deviation. 
The four distributions use the following base MOI configurations:

\begin{align}
\mathbf{J}_{\text{base}}^{(1)} &= [1.0, 2.0, 3.0] \\
\mathbf{J}_{\text{base}}^{(2)} &= [3.0, 1.0, 2.0] \\
\mathbf{J}_{\text{base}}^{(3)} &= [3.0, 2.0, 1.0] \\
\mathbf{J}_{\text{base}}^{(4)} &= [2.0, 3.0, 1.0]
\end{align}

These distributions represent different permutations of the same eigenvalues, creating rigid bodies with varied principal axis orientations. 
This approach challenges models to generalize across different rotational behaviors resulting from the same underlying physics but with different inertial configurations.

\subsection{Simulation Scenarios}\label{subsec:simulation-scenarios}
\label{sec:simulation scenarios}

\paragraph{Freely Rotating}
Freely rotating objects are governed by kinematic equations. Without any external torques, they reduce to the following: 
\begin{equation}
\bJ \dot{\bomega} = - \bomega \times \bJ \bomega
\end{equation}

While this equation is simplified and maintains energy conservation, the gyroscopic torque defined through the cross-product is highly non-linear and can lead to chaotic motion depending on the mass distribution.

\paragraph{Linear Control} In the case of having an identity moment of inertia matrix, the gyroscopic torque can be neglected, and the kinematic equations
\begin{align}
\bJ \dot{\bomega} + \bomega \times \bJ \bomega = \tau_{\text{ext}}
\end{align}
reduce to:
\begin{align}
\bJ \dot{\bomega} = \tau_{\text{ext}}
\end{align}

To this end, we simulate a linear controller where all acceleration is \textit{driven} by an external torque which changes according to the linear control law $\tau_{\text{ext}} = \bJ (\bA \bomega + \bb)$ where $A \in \mathbb{R}^{3\times3}$ and $b \in \mathbb{R}^3$.

By defining $\tau_{ext}$ in this manner, we see that the angular velocity changes according to the external control signal directly as a function of orientation: $\dot{\bomega} = (\bA \bomega + \bb)$, which is a useful simplified external torque without highly non-linear gyroscopic effects.
Notably, the baseline \SO-nCDE and \SO-GRU also perform competitively in this scenario due to it's simplicity.

\paragraph{Configuration Dependent Torque}
We further consider scenarios where the external torque depends explicitly on the current orientation, leading to
\begin{equation}
\bJ \dot{\bomega} + \bomega \times \bJ \bomega = \tau_{\text{ext}}(R)
\end{equation}
where $R \in \SO$ is the current orientation. 
This class of torques arises in various physical systems where the interaction with external fields generates configuration-dependent forcing, such as magnetic dipoles in uniform fields or gravitational gradients. 

Unlike the linear control case, the rotational motion is coupled directly to the orientation state in a global reference frame, potentially leading to multiple equilibria and complex trajectories. 
The external torque takes the general form $\tau_{\text{ext}}(R) = f(R\mathbf{v})$ where $\mathbf{v} \in \mathbb{R}^3$ represents some body-fixed vector and $f: \mathbb{R}^3 \to \mathbb{R}^3$ is a nonlinear function that describes the physical interaction in the world frame.
We observe this scenario to be most challenging for the models as the trajectories can exhibit erratic motion.

To limit chaotic behavior, we additionally apply a damping signal (e.g., simulating friction or other dissipative forces). 
We initialize the damping matrix $\bD$ to be negative definite such that $\lim_{t \to \infty} \|{\bomega}(t)\| = 0$ while the object is also under the effect of the internal gyroscopic motion:

\begin{equation}
\bD = \begin{bmatrix}
    -0.2 & 0 & 0 \\
    0 & -0.2 & 0 \\
    0 & 0 & -0.2
\end{bmatrix}
\end{equation}

These are combined for the following total external torque:
 
\begin{equation}
\tau_{\text{ext}} = w_1\tau_{\text{config}}(R) + w_2\tau_{\text{damp}}(\bomega)
\end{equation}

\noindent where the weights $w_i$ allow for scaling individual contributions. 
Despite combining multiple effects, these scenarios often exhibit more predictable behavior than pure configuration-dependent cases, as damping terms provide a stabilizing effect by dissipating energy from the system.

\subsection{Baseline Model Design}\label{subsec:baseline-design}

We provide additional details regarding the baseline models.
Specifically, how we adapted the \SO-GRU and \SO-nCDE baselines to predict rotations in \SO.
We also detail the essential components of the conservation approach \cite{mason2023learning} that are used to predict rotations directly without confounding pose estimation from images with object dynamics.

\paragraph{\SO-nCDE}
We adapt the neural CDE baseline from \cite{kidger2020neural} to predict rotations in \SO.
Hermit cubic coefficients are directly constructed on the input 9D rotation representation \cite{zhou2019continuity} via backward differences.
Following \cite{kidger2020neural}, the method encodes an initial value $z_0$, which is then integrated forward in time using \textit{torchdiffeq} \cite{torchdiffeq} and Dormand-Prince 4/5 with respect to the constructed spline.
The latent representation is then decoded into the 6D rotation representation \cite{zhou2019continuity,geist2024learning}, and transformed to a rotation matrix with Gram-Schmidt orthonormalization (GSO).

\paragraph{\SO-GRU}
The GRU baseline consists of a multi-layer gated recurrent unit (GRU) RNN with three recurrent units.
The GRU is similarly applied sequentially on each 9D rotation representation, yielding a 6D prediction converted to an element in \SO{} via GSO.

\paragraph{Conservational \cite{mason2023learning}}
The conservation of energy approach of \cite{mason2023learning} is designed to learn a rotation representation directly from images.
A sequence of images is used to estimate the object's momentum, upon which they integrate the kinematics equations (the authors equivalently use the Lie–Poisson formulation) forward in time and obtain estimates of the \SO~state change under energy conservation.
They then reconstruct the predicted image sequence, and estimate predicted rotation accuracy based on image reconstruction.

We observe that the dynamics component of this pipeline is identical to how we simulate the \textit{Free Rotation} scenario in \cref{sec:simulation scenarios}.
However, for forecasting rotations from a trajectory, this approach has several limitations: 
\begin{enumerate}
\item by solving an initial value problem, they can only condition the solve with a single value
\item high sensitivity to the initial momentum estimate 
\item cannot handle non-conservative systems ($\tau_{ext} \neq 0$)
\end{enumerate}

To study this method more generally, we provide a momentum estimate and similarly integrate the equations of motion forward in time, evaluating the accuracy directly via rotational geodesic error (RGE) instead of for a downstream image reconstruction as in \cite{mason2023learning}.
In this way, we decouple the image reconstruction task and the associated image quality metrics (the authors use pixel mean-squared error) from dynamics estimation.
We emphasize that our problem setting encompasses this task; any dynamics module can be plugged directly into an object pose estimator to forecast rotation.

For a fair comparison, we additionally provide results with a ground-truth momentum estimate, observing that our method outperforms \cite{mason2023learning} in all cases but conservational, where the trajectories are parallel to the ground-truth trajectories (they are offset by the noise of the last observation).
The results can be seen in table \cref{table:conservational_GT}.

In practice, obtaining an accurate momentum estimate of an unknown object would be challenging to obtain from a single trajectory as it essentially requires estimating the velocity and mass distribution.
The authors demonstrate this on a limited set of synthetically generated and well-cropped images; however, such an approach does not generalize to unknown mass distributions, let alone objects, limiting applicability in practice.

\begin{table}[t]
\caption{Comparison of prediction errors with the conservational approach \cite{mason2023learning} using a ground truth momentum estimate across different experimental conditions. Results show mean ± standard deviation in rotational geodesic error (RGE). First and second best results are emphasized.}
\resizebox{\textwidth}{!}{
   \setlength{\tabcolsep}{7pt}
   \small
   \centering
  \begin{tabular}{l|lccccc}
    \toprule
    \makecell{Extrapolation \\ Horizon} & Method & \makecell{Freely\\ Rotating} & \makecell{Linear\\ Control} & \makecell{Velocity\\ Damping} & \makecell{Configuration-\\Dependent Torque} & \makecell{Variable\\ Dynamics} \\
    \midrule
\parbox[t]{2mm}{\multirow{2}{*}{\rotatebox[origin=c]{0}{\textit{$t=0.8s$}}}}
 & Conservational (\cite{mason2023learning}) & \textcolor{blue}{$1.15 \pm 0.00$} & \textcolor{blue}{$1.25 \pm 0.00$} & \textcolor{blue}{$1.14 \pm 0.00$} & \textcolor{blue}{$1.26 \pm 0.00$} & \textcolor{blue}{$1.43 \pm 0.00$} \\
  & SG-nCDE (Ours) & \boldmath \textbf{$0.87 \pm 0.05$} & \boldmath \textbf{$0.49 \pm 0.03$} & \boldmath \textbf{$0.42 \pm 0.02$} & \boldmath \textbf{$0.58 \pm 0.03$} & \boldmath \textbf{$0.89 \pm 0.07$} \\
 
\midrule
\parbox[t]{2mm}{\multirow{2}{*}{\rotatebox[origin=c]{0}{\textit{$t=1.2s$}}}}
 & Conservational (\cite{mason2023learning}) & \boldmath \textbf{$1.15 \pm 0.00$} & \textcolor{blue}{$2.05 \pm 0.00$} & \textcolor{blue}{$1.87 \pm 0.00$} & \textcolor{blue}{$2.05 \pm 0.00$} & \textcolor{blue}{$2.44 \pm 0.01$} \\
  & SG-nCDE (Ours) & \textcolor{blue}{$1.28 \pm 0.08$} & \boldmath \textbf{$0.64 \pm 0.03$} & \boldmath \textbf{$0.50 \pm 0.03$} & \boldmath \textbf{$0.74 \pm 0.04$} & \boldmath \textbf{$1.31 \pm 0.14$} \\
    \bottomrule
  \end{tabular}
  \vspace{-4mm}
  }
\label{table:conservational_GT}
\end{table}

\section{Additional Background}\label{sec:appendix-background}

\subsection{Moment of Inertia and Diagonalization}\label{subsec:inertia-diagonalization}
The moment of inertia (MOI) tensor $\mathbf{J}$ is defined as the integral of the mass distribution that characterizes a body's resistance to rotational acceleration about any axis. 
For a continuous mass distribution with density $\eta(\mathbf{r})$:

\begin{align}
J_{ij}
&= \begin{bmatrix}
    J_{xx} & J_{xy} & J_{xz} \\
    J_{xy} & J_{yy} & J_{yz} \\
    J_{xz} & J_{yz} & J_{zz}
\end{bmatrix}
\end{align}

with components:
\begin{align*}
J_{xx} &= \int \eta(\mathbf{r})(y^2 + z^2) \, dV \\
J_{yy} &= \int \eta(\mathbf{r})(x^2 + z^2) \, dV \\
J_{zz} &= \int \eta(\mathbf{r})(x^2 + y^2) \, dV \\
J_{xy} = J_{yx} &= -\int \eta(\mathbf{r})x y \, dV \\
J_{xz} = J_{zx} &= -\int \eta(\mathbf{r})x z \, dV \\
J_{yz} = J_{zy} &= -\int \eta(\mathbf{r})y z \, dV
\end{align*}
where $\mathbf{r} = (x,y,z)$ and $\mathbf{r}^2 = x^2 + y^2 + z^2$. 
Diagonalization yields the principal moments:
\begin{align}
\mathbf{P}^{-1}\mathbf{J}\mathbf{P} &= 
\begin{bmatrix}
    I_1 & 0 & 0 \\
    0 & I_2 & 0 \\
    0 & 0 & I_3
\end{bmatrix}
\end{align}

where $I_k$ are eigenvalues of $\mathbf{J}$ and $\mathbf{P}$ contains the corresponding eigenvectors as columns.

As any MOI tensor can be diagonalized, we consider only diagonal uniform and non-uniform MOI in our numerical simulations, with objects rotating in the canonicalized coordinate system.
We note that \cite{mason2023learning} also consider several fixed non-diagonal MOIs, but these reduce to diagonal cases \cite{bastian2024hybrid}.
Instead, we sample from distributions around non-uniform MOI tensors, which are varied during training, validation, and testing, forcing models to generalize.

\subsection{Exponential and Logarithmic Maps}\label{subsec:exp-log-maps}

\begin{definition}[Exponential Map on SO(3)]
\label{def:exp_map}
The exponential map $\expmap: \so \to \SO$ is a surjective mapping from the Lie algebra $\so$ to the Lie group $\SO$ defined as:
\begin{equation}
\expmap(\boldsymbol{\xi}) = \sum_{n=0}^{\infty} \frac{1}{n!}\boldsymbol{\xi}^n = \mathbf{I} + \boldsymbol{\xi} + \frac{1}{2!}\boldsymbol{\xi}^2 + \frac{1}{3!}\boldsymbol{\xi}^3 + \cdots
\end{equation}
where $\boldsymbol{\xi} \in \so$ is a skew-symmetric matrix. 

Geometrically, if $\boldsymbol{\xi} = \hatop{\mathbf{v}}$ for some $\mathbf{v} \in \R^3$ with $\|\mathbf{v}\| = \theta$, then $\expmap(\boldsymbol{\xi})$ represents a rotation by angle $\theta$ about the axis $\mathbf{v}/\|\mathbf{v}\|$.
\end{definition}

\begin{definition}[Logarithmic Map on SO(3)]
\label{def:log_map}
The logarithmic map $\logmap: \SO \to \so$ is the local inverse of the exponential map, retrieving the corresponding Lie algebra element from a rotation matrix:
\begin{equation}
\logmap(\mathbf{R}) = \boldsymbol{\xi} \in \so
\end{equation}
such that $\expmap(\boldsymbol{\xi}) = \mathbf{R}$. This mapping is well-defined for all $\mathbf{R} \in \SO$ where the rotation angle $\theta$ satisfies $0 \leq \theta < \pi$.
\end{definition}

The logarithmic map $\logmap$ is not uniquely defined for rotations of angle $\theta = \pi$, as rotations by $\pi$ radians about antipodal axes $\mathbf{n}$ and $-\mathbf{n}$ yield identical rotation matrices, yielding a singularity.

\subsection{SO(3) Savitzky-Golay Filtering \cite{jongeneel2022geometric}}\label{subsec:savitzky-golay}
\label{sec:sg_filter_details}

The Savitzky-Golay filter on SO(3) solves a least-squares problem to estimate polynomial coefficients that best fit the noisy rotational data in the Lie algebra, as described in Theorem 10.

The design matrix $\mathbf{A} \in \mathbb{R}^{3(2n + 1) \times 3(p+1) }$ takes the form of a Vandermonde matrix with time-shifted polynomial coefficients.
It is expanded with the Kroeneker product $\otimes$ and $\mathbf{I}_3$, the identity matrix in $\mathbb{R}^3$ such that $\mathbf{A} = \hat{A} \otimes \mathbf{I}_3$.
 $\hat{A} \in \mathbb{R}^{(2n + 1) \times (p + 1)}$ is defined as:

\[
\hat{A} = \begin{bmatrix}
1 & (t_{-n} - t_k)  & \cdots & \frac{1}{p} (t_{-n} - t_k)^p \\
1 & (t_{-n+1} - t_k) & \cdots & \frac{1}{p} (t_{-n+1} - t_k)^p \\
\vdots & \vdots & \vdots & \vdots \\
1 & (t_n - t_k) & \cdots & \frac{1}{p} (t_n - t_k)^p
\end{bmatrix}
\]

$\bb$ is constructed by rotational differences in the lie algebra and has the form:
\[
\bb = \begin{bmatrix}
\logmap(\bxtilde_{k-n}\bxtilde_k^{-1})^\vee \\
\vdots \\
\logmap(\bxtilde_{k+n}\bxtilde_k^{-1})^\vee
\end{bmatrix}
\]

$(\cdot)^\vee: \so \to \R^3$ is the inverse of the hat operator that maps skew-symmetric matrices to vectors.

This formulation brings several advantages over conventional interpolation with splines. 
It filters robustly in the region of a point (with an arbitrary support window) and can be solved without iterative optimization.
Moreover, the derivatives are smooth up to the order of the polynomial, which we leverage in the next section to demonstrate universal approximation based on the results from \cite{morrill2022choice}.

\section{Proof of Proposition 10}\label{sec:appendix-proof}

\label{proof:prop_control_signal}
We demonstrate our modified Savitzky-Golay filter satisfies the three properties of a suitable control path for the neural CDE as described by Morrill et al. \cite{morrill2022choice}.
We first restate the proposition regarding the suitability of the control path.%

\begin{proposition*}[10]
\label{prop:control_signal}
Let $\phit$ be the control path constructed in Def. (7), with polynomial coefficients $\brho$ based on solving the optimization problem in Thm. (8). 
If the maximum rotational difference $\delta_n = \max_k \|\logmap(\bxtilde_k \bxtilde_0^{-1})\|$ within the window is bounded, then:
\begin{enumerate}
    \item $\phit$ is analytic and twice differentiable,
    \item The derivatives $\varphi'(t)$ and $\varphi''(t)$ are bounded,
    \item $\phit$ uniquely minimizes the problem in Thm. (8).
\end{enumerate}\end{proposition*}

\begin{proof}

\paragraph{Smoothness} 
We begin by demonstrating sufficient smoothness of the map  $\phit$. 
We observe that $\bxtildek$ is constant, and the term  $\expmap\left( \bp\left( t - \tk; \brhok \right) \right)$ is a composition of a polynomial and the exponential map.  

We note that the logarithmic map $\logmap: \SO \to \so$ is not uniquely defined for rotations of angle $\pi$.
Specifically, for $\| \xi \| = \pi$ the axis of rotation becomes non-unique, as rotations by  $\pi$  about  $\mathbf{n}$  and  $-\mathbf{n}$ yield the same rotation matrix. 
This requires constraints or assumptions on the magnitude of rotations within any given filter window to ensure smoothness and injectivity of the maps.

Since the exponential map $\expmap: \so \to \SO$ is smooth and analytic on $\so$ \cite{rohan2013exponential}, and the polynomial $\bp\left( t - \tk ; \brhok \right) \in C^{\infty}$  by definition, we have that  $\phit \in C^{\infty}$  provided that  $\| \bp( t - \tk ; \brhok ) \| < \pi$  for all  $t$.

In practice, the magnitude of rotational differences encountered in the filtering process is small, especially when dealing with high-frequency sampling and smooth motion trajectories. 
Therefore, the norm  $\| \bp( t - \tk ; \brhok ) \| < \pi$ holds in most cases. 
However, to rigorously ensure the smoothness and differentiability of $\phit$, we impose a constraint that  $\| \bp( t - \tk ; \brhok ) \| < \pi$  for all  $t$  within the filter window as the exponential map is injective on this domain. 
This can be achieved by appropriately selecting the filter window size and ensuring the sampling rate is sufficiently high with respect to the angular velocity.

Next, we consider smoothness of the derivatives $\frac{d}{dt}\phit$ and $\frac{d^2}{dt^2}\phit$ of $\phit$.
We observe that they coincide with the angular velocity and acceleration, which are defined by \cite{jongeneel2022geometric} as:
\begin{align*}
\widehat{\boldsymbol{\omega}}(t) = \ &D\expmap\left( \boldsymbol{\varphi}(t)^\wedge \right) \dot{\boldsymbol{\varphi}}(t), \\
\quad \widehat{\dot{\boldsymbol{\omega}}}(t) = \ &D^2\expmap\left( \boldsymbol{\varphi}(t)^\wedge \right)[\dot{\boldsymbol{\varphi}}(t), \dot{\boldsymbol{\varphi}}(t)] + \\
\ &D\expmap\left( \boldsymbol{\varphi}(t)^\wedge \right) \ddot{\boldsymbol{\varphi}}(t)
\end{align*}

where $D^kf(\cdot)$ denotes the $k$-th order directional derivative with respect to $f$. 
These are, again, compositions of analytic functions and the polynomial derivatives, which are bounded and continuous and, therefore, continuously differentiable.
We refer to \cite{jongeneel2022geometric} for a complete derivation of the derivatives.

\paragraph{Boundedness}
To show the polynomial coefficients $\brho$ are bounded, we observe that they are calculated via least squares with $\bA$ and $\bb$ defined as in \cref{sec:sg_filter_details}. 
We can see that $\hat{A}$ has full column rank as long as $2n + 1 \geq p + 1$. Therefore the rank of $\bA$ can be obtained via: rank($\bA$) = rank($\hat{A}$) $\cdot$ rank($\bI_3$).
Since both these matrices have full column rank, so does $\bA$, and therefore, the matrix $\bA^T\bA$ is symmetric and positive definite. 
The inverse $(\bA^T\bA)^{-1}$ therefore exists and is bounded.

The data vector \(\bb\) is constructed from the residuals \(\boldsymbol{\delta}_m\), which are bounded by \(\|\boldsymbol{\delta}_m\| \leq \theta_{\max} < \pi\). The entries of \(\bA^\top \bb\) are sums of the form:
\[
[\bA^\top \bb]_i = \sum_{m=k-n}^{k+n} \dfrac{(t_{m} - t_k)^i}{i!} \boldsymbol{\delta}_m,
\]
which are bounded as \(\boldsymbol{\delta}_k\), and the time intervals are bounded.

Combining the above, the polynomial coefficients \(\boldsymbol{\rho}\) satisfy:
\[
\|\boldsymbol{\rho}\| \leq \left\| (\bA^\top \bA)^{-1} \right\| \left\| \bA^\top \bb \right\|
\]

Finally, we note that as long as our learned weight matrix $\bW$ is bounded, this also holds for the weighted adaptation.
$\bW \in \mathbb{R}^{3(2n - 1) \times 3(2n - 1)}$ is a (similarly expanded with $\otimes \bI_3$) diagonal matrix with a weight coinciding to each point in the trajectory window. Then, the bound can be adapted to:
\[
\|\boldsymbol{\rho}\| \leq \left\| (\bA^\top \bW \bA)^{-1} \right\| \left\| \bA^\top \bW \bb \right\|
\]

In practice, we find that we do not need to impose any constraints on $\bW$ and that the network learns to construct the regression weights that are well-behaved.

\paragraph{Uniqueness}
\cite{morrill2022choice} additionally requires the control signal $X$ to be unique. For regular sampling, this follows as cubic splines pass directly through the control points.

In our case, the polynomial $\bp\left( t - \tk ; \brhok \right)$ does not pass through the control points. 
However, we observe that because the matrix $\bA$ has full column rank, the reduced optimization problem in Thm. (8) is strictly convex, and thus the least squares problem attains a unique solution.
In the case of irregular sampling, however, this does not hold \cite[B.2.1]{morrill2022choice}.
As a simple counterexample, consider that a point placed on the estimated trajectory $\phit$ would not alter the solution due to the strict convexity of the cost function in Thm (8).
In practice, we find that the method also behaves well for irregularly sampled measurements.
\end{proof}

\end{document}